\def\year{2020}\relax
\documentclass[letterpaper]{article} 
\usepackage{aaai20}  
\usepackage{times}  
\usepackage{helvet} 
\usepackage{courier}  
\usepackage[hyphens]{url}  
\usepackage{graphicx} 
\usepackage{graphicx}  
\usepackage{amsmath}
\usepackage{amsfonts}
\usepackage{subcaption}
\usepackage{algorithm}
\usepackage{algorithmic}
\usepackage{amsthm}

\usepackage{color}

\newtheorem{theorem}{Theorem}

\def\clip{\operatorname{clip}}
\def\IPO{IPO}
\title{\IPO: Interior-point Policy Optimization under Constraints}

\author{Yongshuai Liu, Jiaxin Ding, Xin Liu\\ 
University of California, Davis\\
\{yshliu, jxding, xinliu\}@ucdavis.edu}

\begin{document}

\maketitle

\begin{abstract}

In this paper, we study reinforcement learning (RL) algorithms to solve real-world decision problems with the objective of maximizing the long-term reward as well as satisfying cumulative constraints. 
We propose a novel first-order policy optimization method, Interior-point Policy Optimization ({\IPO}), 
which augments the objective with logarithmic barrier functions, inspired by the interior-point method. 
Our proposed method is easy to implement with performance guarantees and can handle general types of cumulative
multi-constraint settings. 
We conduct extensive evaluations to compare our approach with state-of-the-art baselines.  
Our algorithm outperforms the baseline algorithms, 
in terms of reward maximization and constraint satisfaction. 
\end{abstract}

\section{Introduction}\label{introduction}

Recent advances have demonstrated significant potentials of deep reinforcement learning (RL) in solving complex sequential decision and control problems, e.g., the Atari game~\cite{mnih2015human}, robotics~\cite{andrychowicz2018learning}, Go~\cite{silver2016mastering}, etc. 
In such RL problems, the objective is to maximize the discounted cumulative reward. In many other problems, in addition to maximizing the reward, a policy needs to satisfy certain constraints. 
For example, 
in a cellular network, a common objective for the network operator is to maximize the throughput or cumulative data transmitted to users.  
At the same time, users may have different requirements for the quality of service,  
such as the requirements on average latency, cumulative throughput, or the average package loss rate,  
which are constraints on the optimization problem~\cite{julian2002qos}.
Consider another example of robot manipulation and control. 
In the task of placing an object~\cite{pham2018optlayer}, the reward is the measurement of how well the object is placed, 
while there are constraints on the motion of the robot arm, such as how much the arm can twist. 

 RL with constraints is usually modeled as a Constrained Markov Decision Process (CMDP)\cite{altman1999constrained}, where the agent must act with respect to constraints, in addition to reward maximization. There are two types of constraints: instantaneous constraints (e.g. robot arm twist) and cumulative constraints (e.g. average latency).  
 An instantaneous constraint is a constraint that the chosen action needs to satisfy in each step.  
A cumulative constraint requires that the sum of one constraint variable from the beginning to the current time step is within a certain limit. 
In this work, we focus on cumulative constraints, including both discounted cumulative constraints and mean valued constraints. 

A common approach to solve CMDPs is the Lagrangian relaxation method~\cite{chow2017risk,tessler2018reward}. 
  The constrained optimization problem is reduced to an unconstrained one by augmenting the objective function with a sum of the constraint functions weighted by their corresponding Lagrange multipliers. Then, the Lagrange multipliers are updated in the dual problem to satisfy the constraints.   
 Although constraints are satisfied
 when the policy converges, this approach is sensitive to the initialization of the Lagrange multipliers and the  learning rate, 
 and the policy obtained during training does not consistently satisfy the constraints, 
 as discussed in~\cite{achiam2017constrained,chow2019lyapunov}. 

Constrained policy optimization (CPO)~\cite{achiam2017constrained} is proposed to solve CMDPs. It extends the trust-region policy optimization (TRPO) algorithm~\cite{schulman2015trust} to handle the constraints. CPO monotonically improves the policy during training, demonstrating promising empirical performance, 
and it guarantees constraint satisfaction during the training process once the constraints are satisfied~\cite{chow2019lyapunov}. 
However, CPO needs to calculate the second-order derivatives and thus is complicated to compute and implement. In addition, CPO does not handle mean valued constraints~\cite{tessler2018reward}, and it is difficult to employ CPO when there are multiple constraints. 


In this paper, we propose a first-order optimization method, Interior-point Policy Optimization ({\IPO}),
 to solve CMDPs with different types of cumulative constraints.  
Specifically, inspired by the interior-point method~\cite{boyd2004convex}, we augment the objective function of {\IPO} with logarithmic barrier functions as penalty functions to accommodate the constraints.  
Intuitively, we would like to construct 
functions such that 1)
if a constraint is satisfied, the penalty added to the reward function is zero, and 2)
 if the constraint is violated, the penalty goes to negative infinity. 
The logarithmic barrier functions satisfy these requirements, are easy to implement, and also provide nice analytical properties. 
For policy optimization, we leverage PPO~\cite{schulman2017proximal}, and thus inherit its trust region property. We note that other policy optimization algorithms can be integrated when needed, which increases the flexibility of the proposed methodology. 
Our algorithm is easy to implement and the hyperparameters are convenient to tune. 


In summary, our contributions are as follows:
\begin{itemize}
    \item We propose {\IPO}, a first-order optimization RL algorithm under cumulative constraints. The algorithm is easy-to-implement, can handle different types and multiple constraints, with easy-to-tune hyperparameters.
    \item We provide the performance bound of the {\IPO} in terms of reward functions using primal-dual analysis. 
    \item We conduct extensive experiments to compare  {\IPO}  with the Lagrangian relaxation method and CPO on continuous control tasks, such as MuJoco and grid-world in the robotics setting.  {\IPO} outperforms the state-of-art methods with higher long-term reward and lower cumulative constraint values. 
    
\end{itemize}

%

\section{Related work}\label{related-work}
Reinforcement learning with constraints is a significant and challenging topic.  A comprehensive overview can be found in ~\cite{garcia2015comprehensive,dulac2019challenges}. 

The Lagrangian relaxation method is widely applied to solve the RL with constraints. 
Primal-Dual Optimization (PDO) \cite{chow2017risk}  employs the Lagrangian relaxation method to devise policy gradient and actor-critic algorithm for risk-constrained RL. 
PDO is further adapted to off-line policy learning in  \cite{liang2018accelerated} aiming to accelerate the learning process. They use off-line data to pre-train the Lagrange multipliers and reduce the iterations of online updating.
A batch policy learning based on the Lagrangian method is proposed in \cite{le2019batch} which considers both sampling efficiency and constraint satisfaction challenges. 
Reward Constrained Policy Optimization (RCPO)~\cite{tessler2018reward} is proposed as a multi-timescale actor-critic approach. They take advantage of TD-learning to update the policy and handle mean valued constraints based on the Lagrangian method. 

Differing from above methods which tunes Lagrange multipliers in primal and dual space, Constrained Policy Optimization (CPO)~\cite{achiam2017constrained}  uses new approximation methods from scratch to compute the Lagrange multiplier and enforce constraints in each iteration during the training process. 

Another sort of algorithms leverage Lyapunov functions~\cite{khalil2002nonlinear,neely2010stochastic} to handle constraints. 
In \cite{chow2018lyapunov,chow2019lyapunov}, safe approximation policy and value iteration algorithms are induced by Lyapunov constraints. 

There are also works on adding a constrained layer to the policy network to satisfy zero-constraint violation at each time step, such as in \cite{dalal2018safe,pham2018optlayer} . 

To the best of our knowledge, there are no previous works using the interior-point method to solve RL problems with constraints.

\section{Preliminaries}\label{pre}
\subsection{Markov Decision Process}
A Markov Decision Process (MDP) is represented by a tuple $\left(\mathit{S},\mathit{A}, \mathit{R},\mathit{P},\mathit{\mu}, \mathit{\gamma} \right )$ \cite{sutton2018reinforcement}, where $\mathit{S}$ is the set of states, $\mathit{A}$ is the set of actions, $\mathit{R}:\mathit{S} \times \mathit{A} \times \mathit{S} \mapsto \mathbb{R}$ is the reward function, $\mathit{P}:\mathit{S} \times \mathit{A} \times \mathit{S} \mapsto [0,1]$ is the transition probability function, where $\mathit{P}(s^{'}|s,a)$ is the transition probability from state $s$ to state $s^{'}$ with taking action $a$, $\mathit{\mu}: \mathit{S}\mapsto [0,1]$ is the initial state distribution and $\mathit{\gamma}$ is the discount factor for future reward. A policy $\pi :\mathit{S} \mapsto \mathcal{P}(\mathit{A})$ is a mapping from states to a probability distribution over actions and $\pi(a|s)$ is the probability of taking action $a$ in state $s$. We write a policy $\pi$ as $\pi_{\theta}$ to emphasize its dependence on the parameter $\theta$ (e.g., a neural network policy with parameter $\theta$). A common goal of a MDP is to select a policy $\pi_{\theta}$ which maximizes the discounted cumulative reward. It is denoted as
\begin{equation}\label{objective}
    \max_{\theta}~J_{R}^{\pi_{\theta}} = \mathbb{E}_{\tau \sim \pi_{\theta}}[\sum_{t=0}^{\infty}\mathit{\gamma}^{t}\mathit{R}(s_{t},a_{t},s_{t+1})]
\end{equation}
where $\tau = (s_{0}, a_{0},s_{1}, a_{1}... )$ denotes a trajectory, and $\tau \sim \pi_{\theta}$ means that the distribution over trajectories is following policy $\pi_{\theta}$.

For a trajectory starting from state $s$,  the value function of state $s$ is
\begin{equation*}
    V_{R}^{\pi_{\theta}}{(s)} = \mathbb{E}_{\tau \sim \pi_{\theta}}[\sum_{t=0}^{\infty}\gamma^{t}\mathit{R}(s_{t},a_{t},s_{t+1})|s_{0}=s]
\end{equation*}
The  action-value function of state $s$ and action $a$ is 
\begin{equation*}
    Q_{R}^{\pi_{\theta}}{(s,a)} = \mathbb{E}_{\tau \sim \pi_{\theta}}[\sum_{t=0}^{\infty}\gamma^{t}\mathit{R}(s_{t},a_{t},s_{t+1})|s_{0}=s, a_{0}=a]
\end{equation*}
and the advantage function is 
\begin{equation}\label{advantage}
    A_{R}^{\pi_{\theta}}{(s,a)} = Q_{R}^{\pi_{\theta}}{(s,a)} - V_{R}^{\pi_{\theta}}{(s)}
\end{equation}

\subsection{Constrained Markov Decision Process}
A Constrained Markov Decision Process (CMDP) extends the MDP by introducing a cost function $\mathit{C}:\mathit{S} \times \mathit{A} \times \mathit{S} \mapsto \mathbb{R}$ (similar to the reward function) for each transition tuple $(s, a,s^{'})$, several constraints  $\widetilde{C_{i}} =  f(\mathit{C}(s_{0},a_{0},s_{1})),...,\mathit{C}(s_{n},a_{n},s_{n+1}))$, and constraint limits $\epsilon_{1},...,\epsilon_{m}$. 
The constraints include discounted cumulative constraints, and mean valued constraints~\cite{altman1999constrained} . 
The expectation over a constraint is defined as:
\begin{equation*}
    J_{C_{i}}^{\pi_{\theta}}  = \mathbb{E}_{\tau \sim \pi_{\theta}} [\widetilde{C_{i}}]
\end{equation*}
The discounted cumulative constraint is in the form of:
\begin{equation}\label{discounted}
    J_{C_{i}}^{\pi_{\theta}} = \mathbb{E}_{\tau \sim \pi_{\theta}}[\sum_{t=0}^{\infty}\gamma^{t}\mathit{C}(s_{t},a_{t},s_{t+1})]
\end{equation}
The mean valued constraint is in the form of:
\begin{equation}\label{mean}
    J_{C_{i}}^{\pi_{\theta}} = \mathbb{E}_{\tau \sim \pi_{\theta}}[\frac{1}{T}\sum_{t=0}^{T-1}\mathit{C}(s_{t},a_{t},s_{t+1})]
\end{equation}
where $T$ is the total number of time steps in each trajectory.

For a CMDP, the goal is to find a policy $\pi_{\theta}$ which maximizes the discounted cumulative reward while satisfying the cumulative constraints. It is denoted as
\begin{equation}\label{objective_constraint}
    \begin{split}
    \max_{\theta} J_{R}^{\pi_{\theta}}\\
    s.t.~~~J_{C_{i}}^{\pi_{\theta}} \leq \epsilon_{i} 
\end{split}
\end{equation}
where $\epsilon_{i}$ is the limit for each constraint.

\subsection{Policy Gradient Methods}
Policy gradient~\cite{sutton2000policy} is a method for finding an optimal policy of a MDP problem. 
It first calculates gradient of the objective Eq. (\ref{objective}),  
\begin{equation*}
    \bigtriangledown J_{R}^{\pi_{\theta}} = \mathbb{E}_{t}[\bigtriangledown_{\theta} log\pi_{\theta}(a_{t}|s_{t})A_{t}]
\end{equation*}
where $\pi_{\theta}$ is the current policy under parameter $\theta$ and $A_{t}$ is the advantage function Eq. (\ref{advantage}) at time step $t$. Thereafter, $\theta$ is updated as 
\begin{equation*}
    \theta = \theta + \eta \bigtriangledown J_{R}^{\pi_{\theta}},
\end{equation*}
where $\eta$ is the learning rate. 


Trust Region Policy Optimization (TRPO)~\cite{schulman2015trust} is proposed to achieve monotonic improvement when updating policy. 
The objective is approximated with a surrogate function, and the step size is limited by the Kullback Leibler (KL) divergence~\cite{kullback1951information}, shown as follows. 
\begin{equation*}\label{trpo}
    \centering
    \begin{split}
     \max_{\theta}~L^{TRPO}(\theta) = \mathbb{E}_{t}[\frac{\pi_{\theta}(a_{t}|s_{t})}{\pi_{\theta_{old}}(a_{t}|s_{t})}A_{t}]\\
    s.t.~~~\mathbb{E}_{t}[KL[\pi_{\theta_{old}}(a_{t}|s_{t}), \pi_{\theta}(a_{t}|s_{t})]] \leq \delta .
\end{split}
\end{equation*}
where $\delta$ is the step size limitation.

TRPO can be approximately solved with a conjugate gradient optimization, which is efficient. 


  Proximal Policy Optimization (PPO)~\cite{schulman2017proximal} approximates the objective by a first-order surrogate optimization problem to reduce the complexity of TRPO, defined as 

\begin{equation}
   \begin{aligned} \label{PPO}
  \max_{\theta}~&L^{CLIP}(\theta )  \\
 = &~\mathbb{E}_{t}[\min(r_{t}(\theta)A_{t},
\clip(r_{t}(\theta),1,1-\epsilon ,1+\epsilon)A_{t})], 
\end{aligned} 
\end{equation}
where $r_t(\theta) = \frac{\pi_{\theta}(a_{t}|s_{t})}{\pi_{\theta_{old}}(a_{t}|s_{t})}$, $A_{t}$ is the advantage function, $\clip(\cdot)$ is the clip function and $r_{t}(\theta)$ is clipped between $\left [ 1-\epsilon, 1+\epsilon  \right ]$. 

\begin{figure}[t]
    \centering
    \includegraphics[width=0.9\columnwidth]{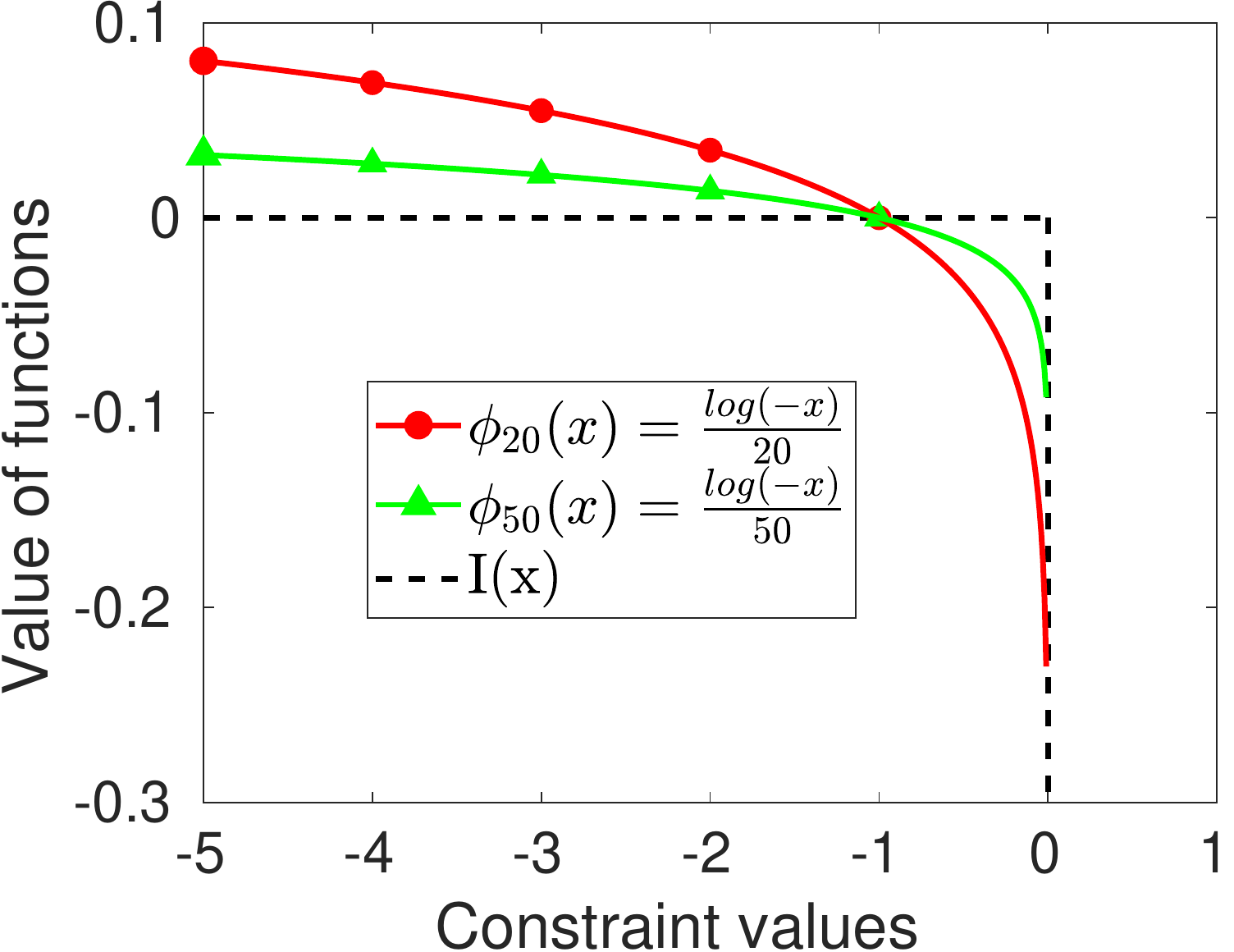}
    \vspace{1pt}
    \caption{Value of indicator function $I(x)$ and logarithmic barrier functions $\phi(x) = \frac{-log(-x)}{t}$. The dashed line is the indicator function and two solid lines are logarithmic barrier function with $t=20$ and $t=50$. We get better approximation with higher $t$ comparing these two solid lines.
    }
    \label{fig:indicator}
\end{figure}

\section{Interior-point Policy Optimization}\label{method}
Now we introduce our Interior-point Policy Optimization ({\IPO}) to solve CMDP. 
We employ the clipped surrogate objective of PPO in Eq. (\ref{PPO}) as our objective, and
augment it with the logarithmic barrier function for the constraints in the interior-point method.  

Our problem is defined as 
\begin{equation}\label{PPO_constraints}
    \begin{split}
    \max_{\theta}~L^{CLIP}(\theta ), \\
    s.t.~~~J_{C_{i}}^{\pi_{\theta}} \leq \epsilon_{i} .
\end{split}
\end{equation}

\subsection{Logarithmic Barrier Function}

Now we denote $\widehat{J}_{C_{i}}^{\pi_{\theta}} = J_{C_{i}}^{\pi_{\theta}} - \epsilon_{i}$, to simplify the notation. 
Our constrained optimization problem can be reduced to an unconstrained one 
by augmenting the objective with indicator functions $I(\widehat{J}_{C_{i}}^{\pi})$,  for each constraint $\widehat{J}_{C_{i}^{\pi}}$ satisfying 
\begin{equation*}
    I{(\widehat{J}_{C_{i}}^{\pi_{\theta}})} = \begin{cases}
 & 0\ \ \ \ \  \ \widehat{J}_{C_{i}}^{\pi_{\theta}}\leq 0,\\ 
 & -\infty\ \ \widehat{J}_{C_{i}}^{\pi_{\theta}}>  0. 
\end{cases}
\end{equation*}
It means that
when the constraints are satisfied, we solve the problem as an unconstrained policy optimization problem only considering the reward; 
however, when any constraint is violated, we must adjust the policy to satisfy the constraint first, since the penalty is $-\infty$. 

The logarithm barrier function is a differentiable approximation of the indicator function, defined as
\begin{equation*}\label{log}
    \phi {(\widehat{J}_{C_{i}}^{\pi_{\theta}})} = \frac{\log(-\widehat{J}_{C_{i}}^{\pi_{\theta}})}{t},
\end{equation*}
where $t>0$ is a  hyperparameter. 
The larger $t$ is, the better the approximation is to the indicator function, as shown in Figure \ref{fig:indicator}. 

Now our objective becomes 
\begin{equation} \label{IPO}
    \max_{\theta}~L^{IPO}(\theta), 
\end{equation}
where $$L^{IPO}(\theta)=L^{CLIP}(\theta )+\sum_{i=1}^{m}  \phi {(\widehat{J}_{C_{i}}^{\pi_{\theta}})}.$$

Thereafter, we can perform first order optimization (e.g. Adam optimizer) to update the parametric policy (e.g. neural network). The pseudo-code of {\IPO} is shown in Algorithm \ref{alg:implementation}.
\begin{algorithm}[t]
\caption{The procedure of {\IPO}} \label{alg:implementation}
\hspace*{0.02in} {\bf Input:} 
Initialize policy $\pi$ with parameter $\theta=\theta_{0}$. Set the hyperparameter  $r$ for PPO clip rate and $t$ for logarithmic barrier function\\
\hspace*{0.02in} {\bf Output:}
The policy parameters $\theta$
\begin{algorithmic}[1]
\STATE Initialize the computational graph structure.
\FOR {iteration k=0,1,2,...}
\STATE Sample N trajectories $\tau_{1}, ..., \tau_{N}$ including observations, actions, rewards and  costs under the current policy $\theta_{k}$ 
\STATE Process the trajectories to advantages, constraint values, etc
\STATE Update the policy parameter with first order optimizer $\theta_{k+1} = \theta_{k} +  \alpha  \bigtriangledown_{\theta} L^{IPO}(\theta) $ where $\alpha$ is learning rate based on the processed trajectories.
\ENDFOR
\RETURN policy parameters $\theta=\theta_{k+1}$
\end{algorithmic}
\end{algorithm}

\subsection{Performance Guarantee Bound} 


\begin{theorem}
\label{thm:IPO}
  The maximum gap between the optimal value of the constrained optimization problem in Eq. (\ref{PPO_constraints})  and the objective of {\IPO} in Eq. 
  (\ref{IPO}) is bounded by $\frac{m}{t}$, where $m$ is the number of constraints and $t$ is the hyperparameter of logarithmic barrier function, if the optimal policy is strictly feasible.
\end{theorem}

\begin{proof}
To be consistent with the standard optimization problem, we first convert our maximization problem to 
a minimization problem by obtaining its negation. 
The problem defined in Eq. (\ref{PPO_constraints})  is now 
\begin{equation}\label{objective-zero}
    \begin{split}
    \min_{\theta} -L^{CLIP}(\theta), \\
    s.t.~~~\widehat{J}_{C_{i}}^{\pi_{\theta}} \leq 0. 
\end{split}
\end{equation}
The objective of {\IPO} in Eq. (\ref{IPO}) becomes
\begin{equation} \label{IPO_zero}
    \min_{\theta} -L^{CLIP}(\theta )-\sum_{i=1}^{m}\frac{log(-\widehat{J}_{C_{i}}^{\pi_{\theta}})}{t}. 
\end{equation}
The Lagrangian function of Eq. (\ref{objective-zero}) is
\begin{equation}\label{primal}
    L(\theta, \lambda_{i}) = -L^{CLIP}(\theta ) + \sum_{i=1}^{m} \lambda_{i}\widehat{J}_{C_{i}}^{\pi_{\theta}},
\end{equation}
where $\lambda_{i} \geq 0$ is the Lagrange multiplier. 

The dual function is
\begin{equation}\label{dual}
    g(\lambda_{i}) = \underset{\theta}{\min} -L^{CLIP}(\theta ) + \sum_{i=1}^{m} \lambda_{i}\widehat{J}_{C_{i}}^{\pi_{\theta}}. 
\end{equation}

If the problem is strictly feasible which means an optimal parameter $\theta^{*}$  for Eq. (\ref{IPO_zero}) exists and $\widehat{J}_{C_{i}}^{\pi_{\theta^*}} < 0$.  
The optimal parameter $\theta^*$ 
must satisfy
\begin{equation}\label{optimal_consition}
    -\bigtriangledown L^{CLIP}(\theta^{*}) + \sum_{i=1}^{m} \frac{1}{-t \times \widehat{J}_{C_{i}}^{\pi_{\theta^{*}}}} \bigtriangledown \widehat{J}_{C_{i}}^{\pi_{\theta^{*}}}=0. 
\end{equation}





We set
\begin{equation}\label{lamda}
    \lambda_{i}^{*} = -\frac{1}{t\times\widehat{J}_{C_{i}}^{\pi_{\theta^{*}}}}, 
\end{equation}
 and plug  $\lambda_{i}^{*}$ into Eq. (\ref{optimal_consition}). We obtain 
\begin{equation}
    -\bigtriangledown L^{CLIP}(\theta^{*}) + \sum_{i=1}^{m} \lambda_{i}^{*} \bigtriangledown \widehat{J}_{C_{i}}^{\pi_{\theta^{*}}}=0
\end{equation}
It means that $\theta^{*}$ minimizes the Lagrangian Eq. (\ref{primal}) under $\lambda_{i} = \lambda_{i}^{*}$. 
That is, 
\begin{equation}
\begin{aligned}
 g(\lambda_{i}^{*}) & = -L^{CLIP}(\theta^{*}) + \sum_{i=1}^{m} \lambda_{i}^{*}\widehat{J}_{C_{i}}^{\pi_{\theta^{*}}}\\
    &= -L^{CLIP}(\theta^{*}) - \frac{m}{t}
\end{aligned}
\end{equation} 
Let $p^*$ be the optimal value in the problem Eq. (\ref{objective-zero}). By the property of duality gap, 
$p^* \geq g(\lambda^*).$ Therefore, 

\begin{equation}
    -L^{CLIP}(\theta^{*}) -  p^{*} \leq \frac{m}{t}. 
\end{equation}
It means the gap between the optimal value of the original constrained problem with clipped surrogate function (Eq. (\ref{PPO_constraints})) and {\IPO} (Eq. (\ref{IPO})) is bounded by $\frac{m}{t}$.
\end{proof}



Theorem \ref{thm:IPO} indicates that a larger $t$ provides a better approximation of the original objective. 
Empirically, we notice that a larger $t$ can lead to a higher reward and cost, but at a lower convergence rate. 
This monotonicity enables us to employ a binary search algorithm to find a $t$ to balance the convergence rate and optimization.

\begin{figure}[t]
     \centering
     \begin{minipage}{\columnwidth}
        \centering
        Point Gather with discounted cumulative constraint
    \end{minipage}
     \vspace{5pt}
     \begin{subfigure}[t]{0.47\columnwidth}
         \centering
         \includegraphics[width=\textwidth]{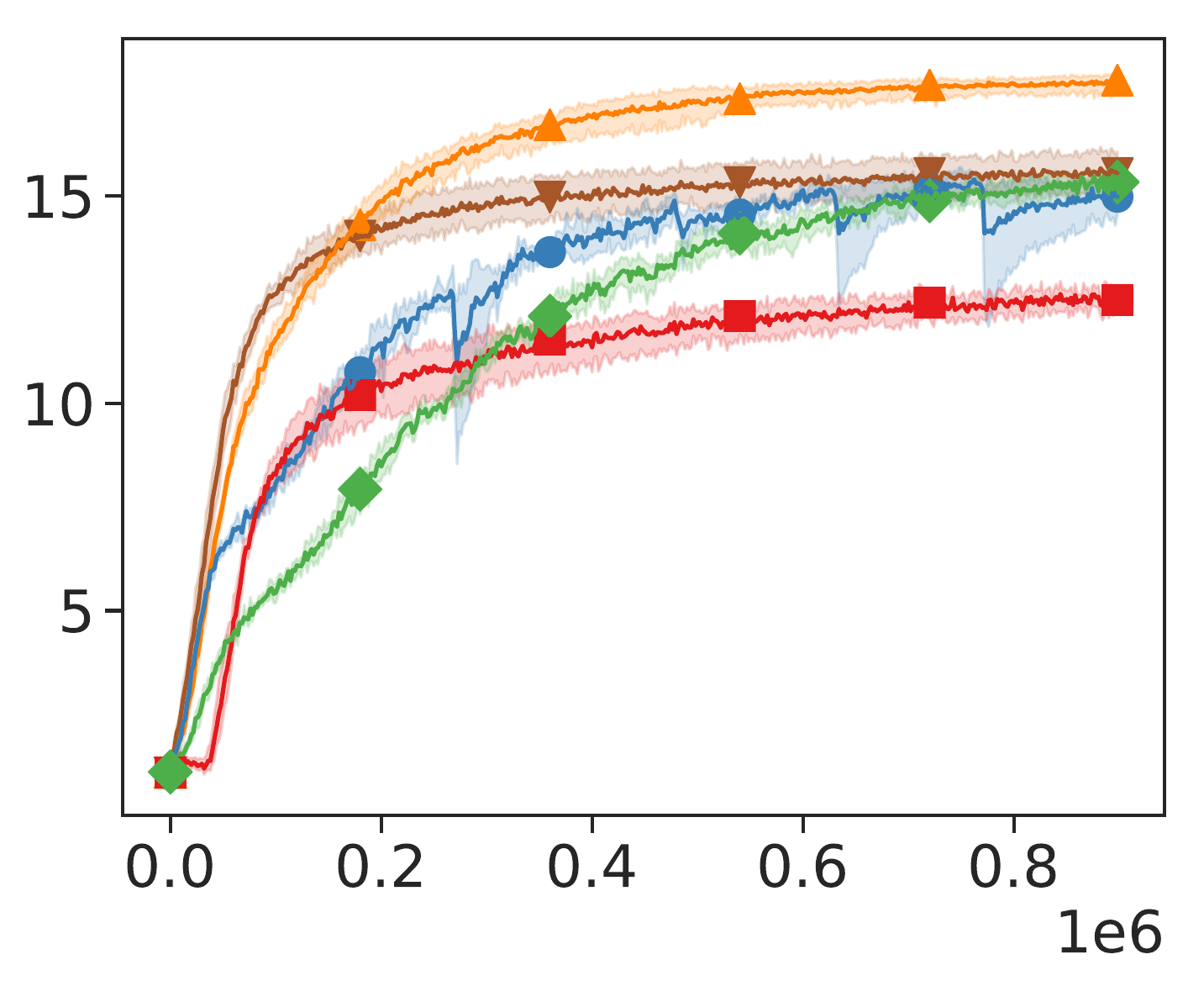}
         \caption{Reward}
         \label{fig:gather_reward}
     \end{subfigure}
     \hfill
     \begin{subfigure}[t]{0.47\columnwidth}
         \centering
         \includegraphics[width=\textwidth]{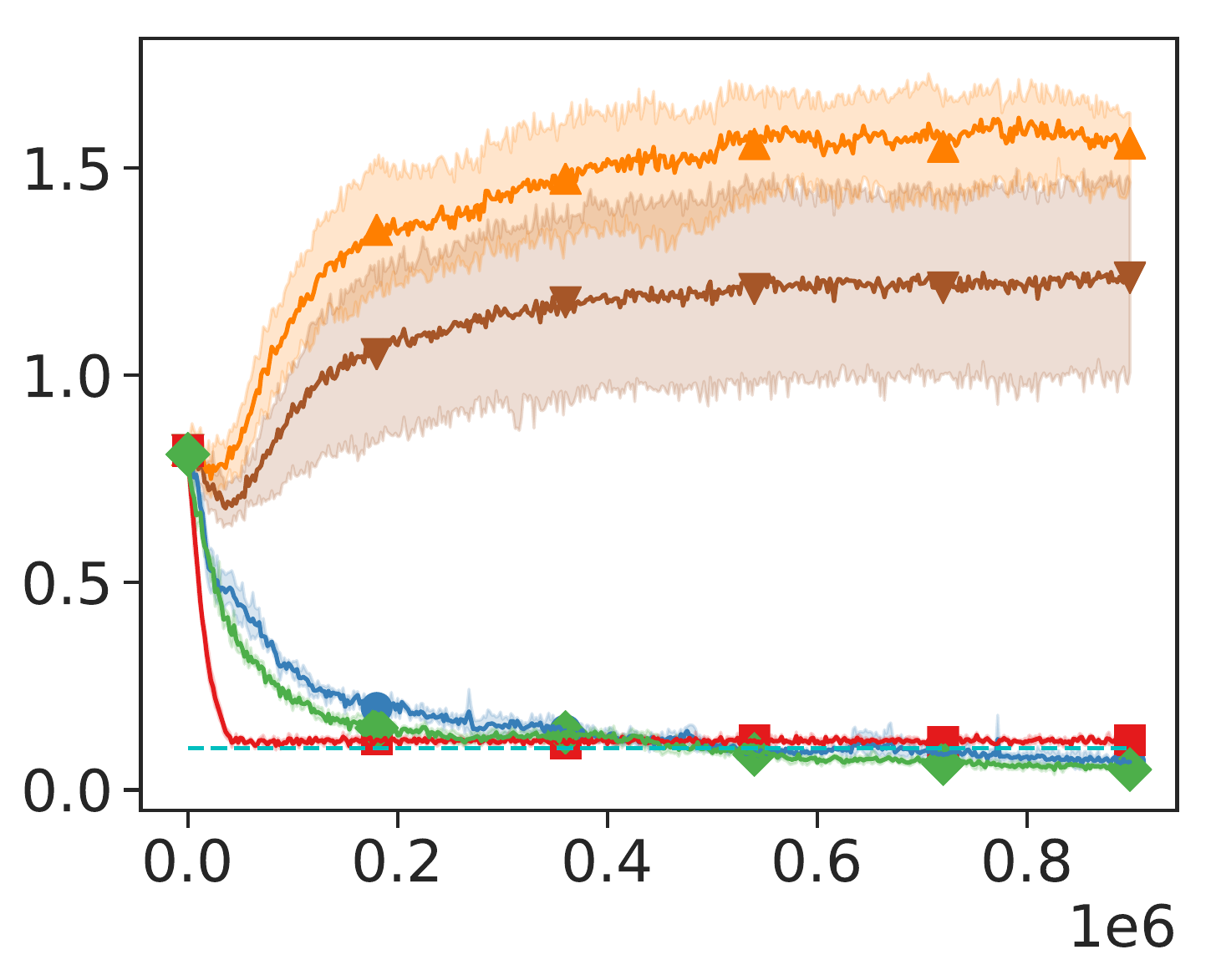}
         \caption{Constraint}
         \label{fig:gather_cost}
     \end{subfigure}
     
     \begin{minipage}{\columnwidth}
        \centering
        Point Circle with discounted cumulative constraint
    \end{minipage}
     \vspace{5pt}
     \begin{subfigure}[t]{0.47\columnwidth}
         \centering
         \includegraphics[width=\textwidth]{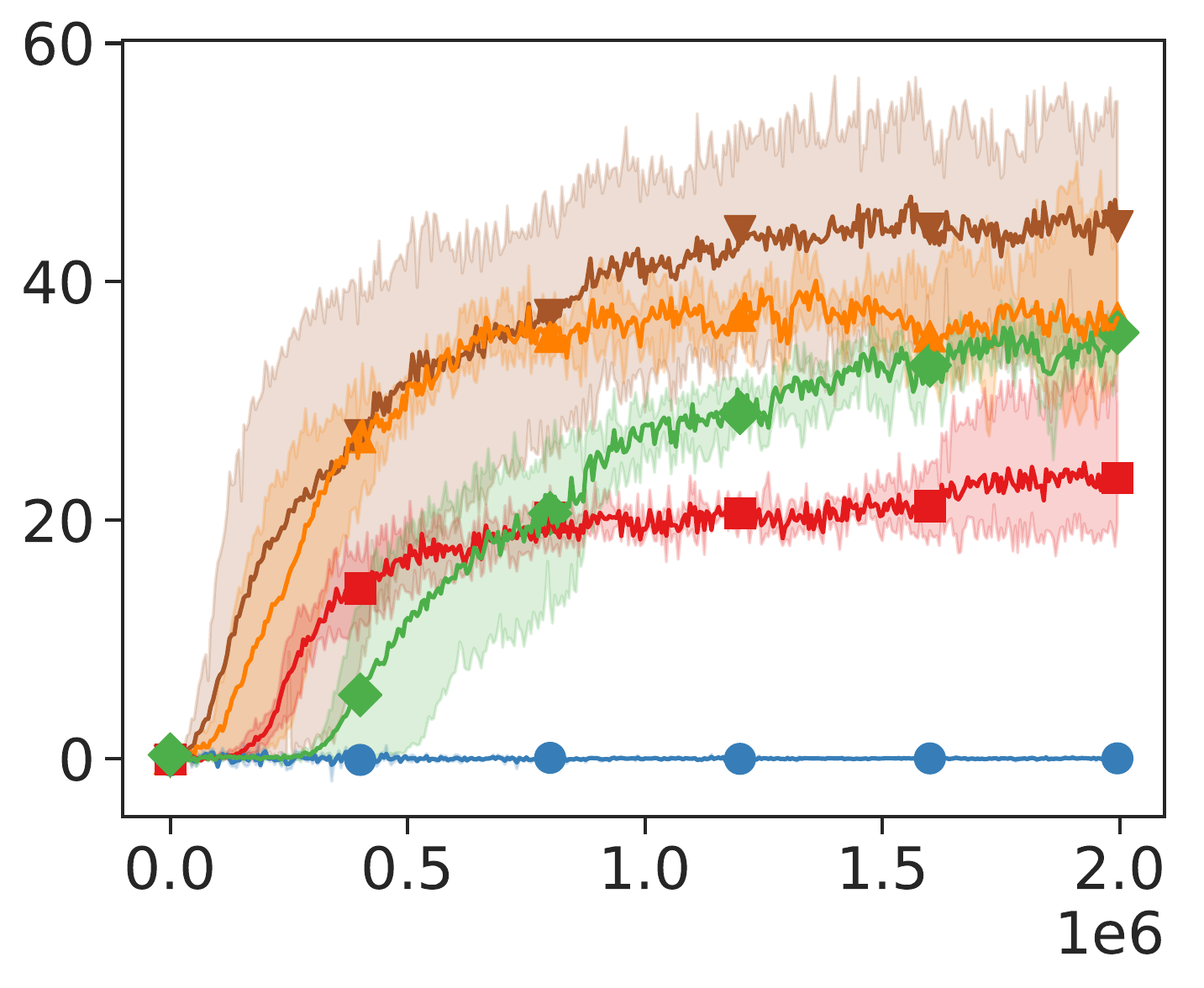}
         \caption{Reward}
         \label{fig:circle_reward}
     \end{subfigure}
     \hfill
     \begin{subfigure}[t]{0.47\columnwidth}
         \centering
         \includegraphics[width=\textwidth]{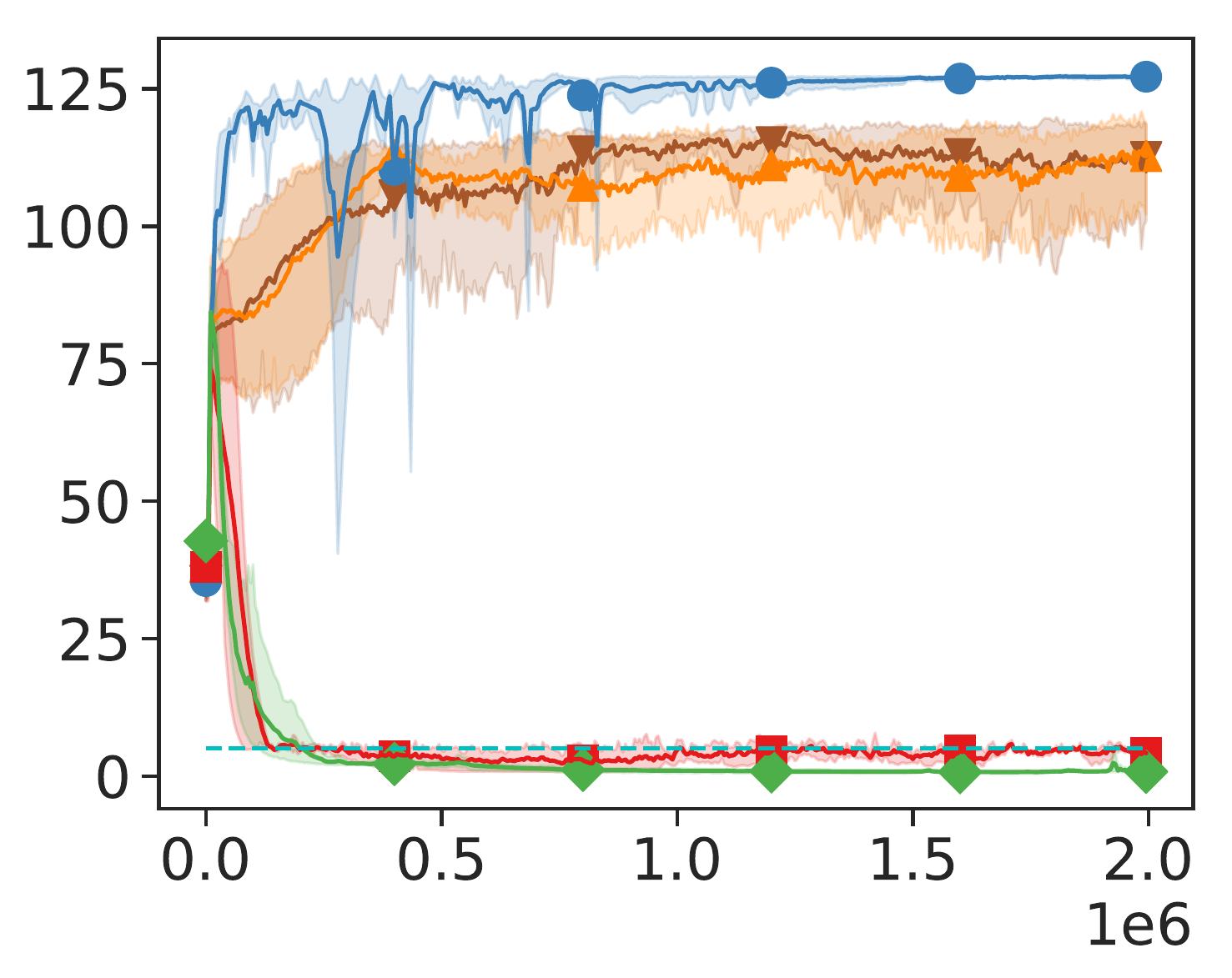}
         \caption{Constraint}
         \label{fig:circle_cost}
     \end{subfigure}
     
    \begin{minipage}{\columnwidth}
        \centering
        HalfCheetah-Safe with discounted cumulative constraint
    \end{minipage}
    \vspace{3pt}
     \begin{subfigure}[t]{0.47\columnwidth}
         \centering
         \includegraphics[width=\textwidth]{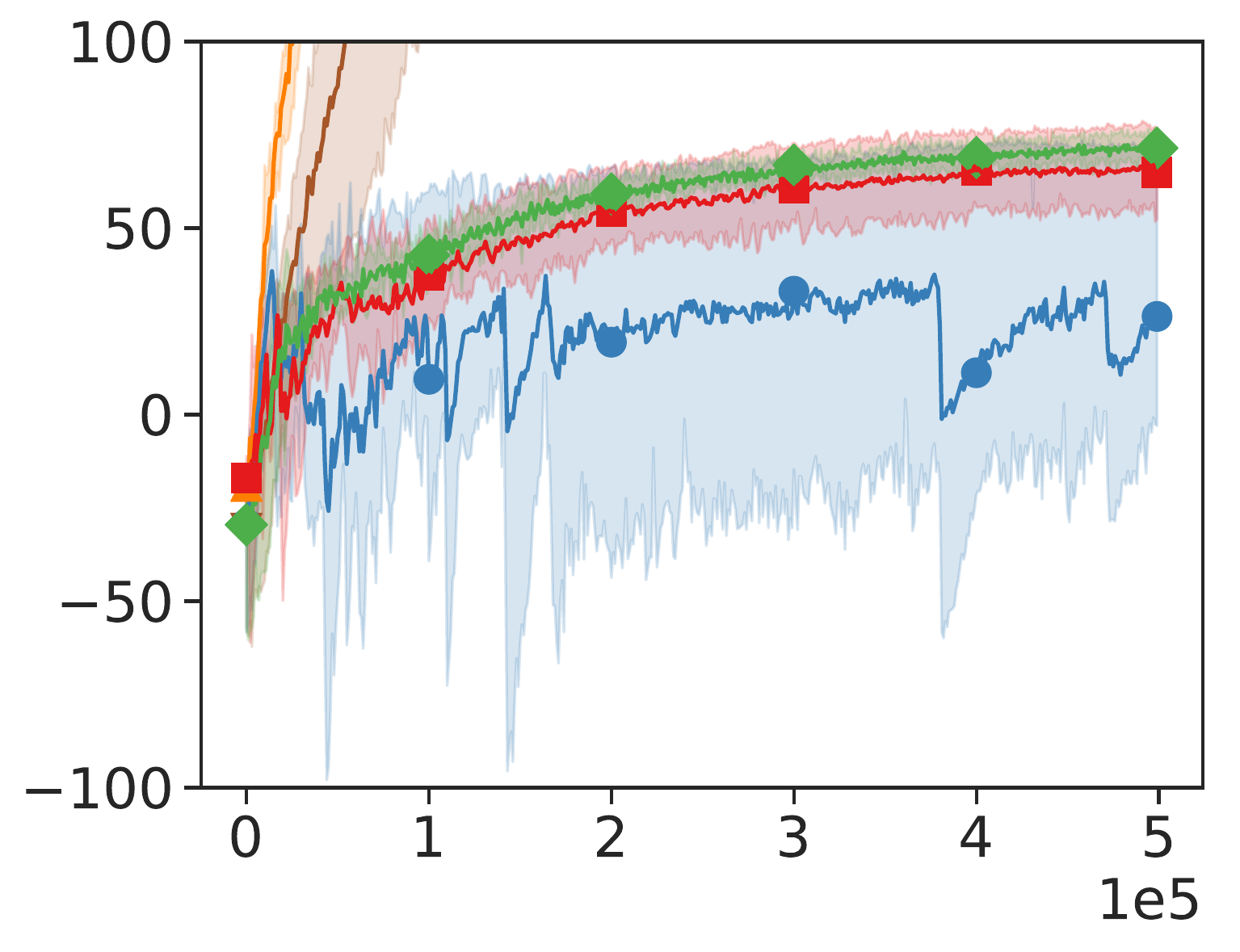}
         \caption{Reward}
         \label{fig:halfcheetah_reward}
     \end{subfigure}
     \hfill
     \begin{subfigure}[t]{0.47\columnwidth}
         \centering
         \includegraphics[width=\textwidth]{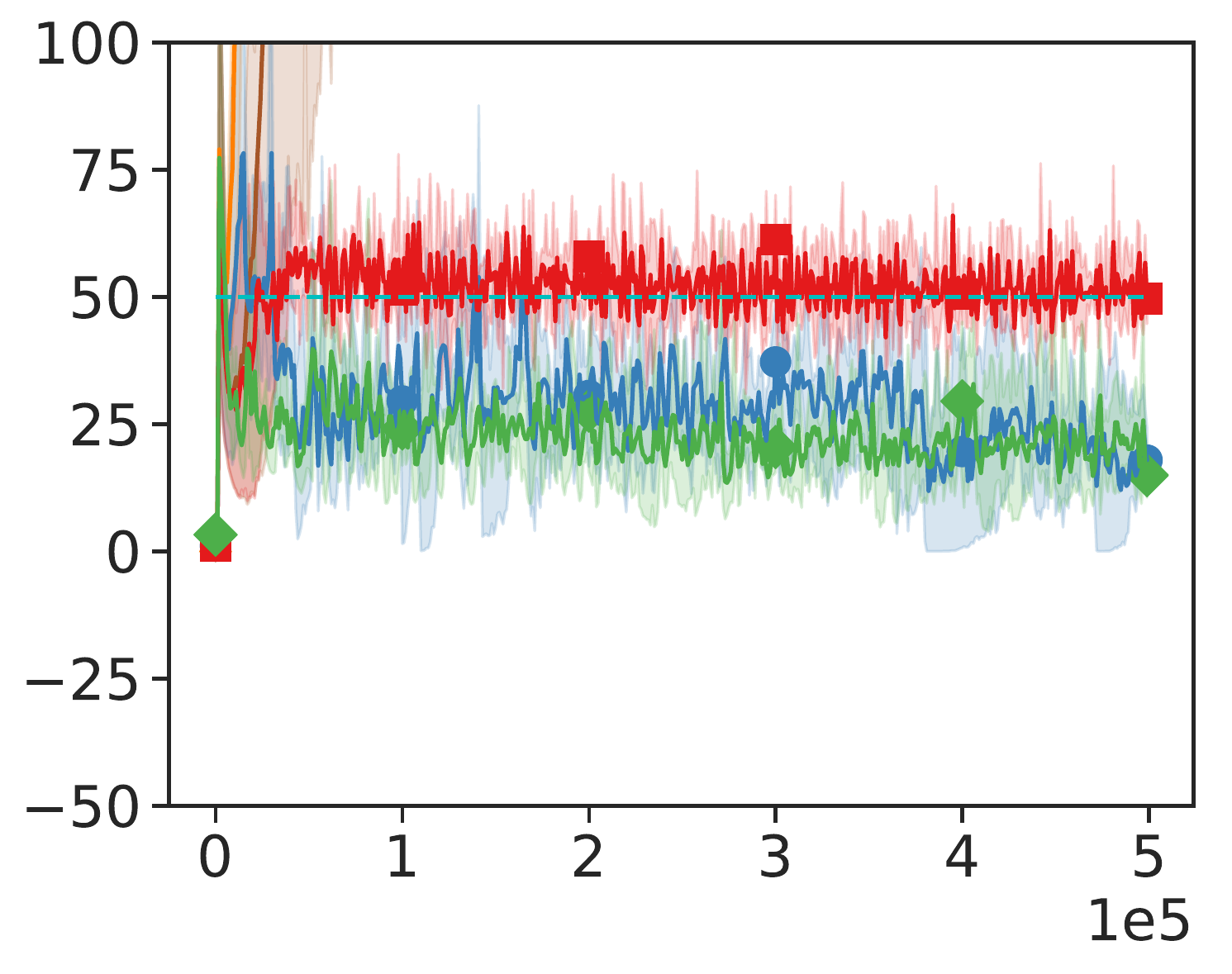}
         \caption{Constraint}
         \label{fig:halfcheetah_cost}
     \end{subfigure}
    \begin{minipage}[t]{0.9\columnwidth}
         \centering
         \includegraphics[width=\textwidth]{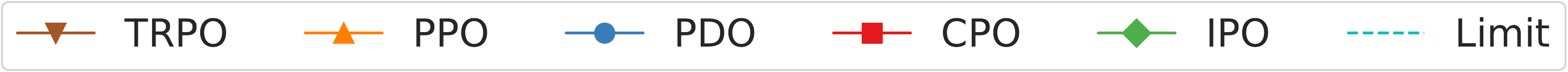}
    \end{minipage}
     \caption{Average performance of TRPO, PPO, PDO, CPO and IPO under Point Gather, Point Circle and HalfCheetah-Safe with discounted cumulative constraints. The x-axis is the number of trajectories. The dashed lines are constrained limits for different tasks which is $0.1$ for Point Gather, $5$ for Point Circle and $50$ for HalfCheetah-Safe.}
     \label{fig:discounted}
\end{figure}

\begin{figure}[t]
     \centering
     \begin{minipage}{\columnwidth}
        \centering
        Point Gather with mean valued constraint
    \end{minipage}
    \vspace{5pt}
     \begin{subfigure}[t]{0.47\columnwidth}
         \centering
         \includegraphics[width=\textwidth]{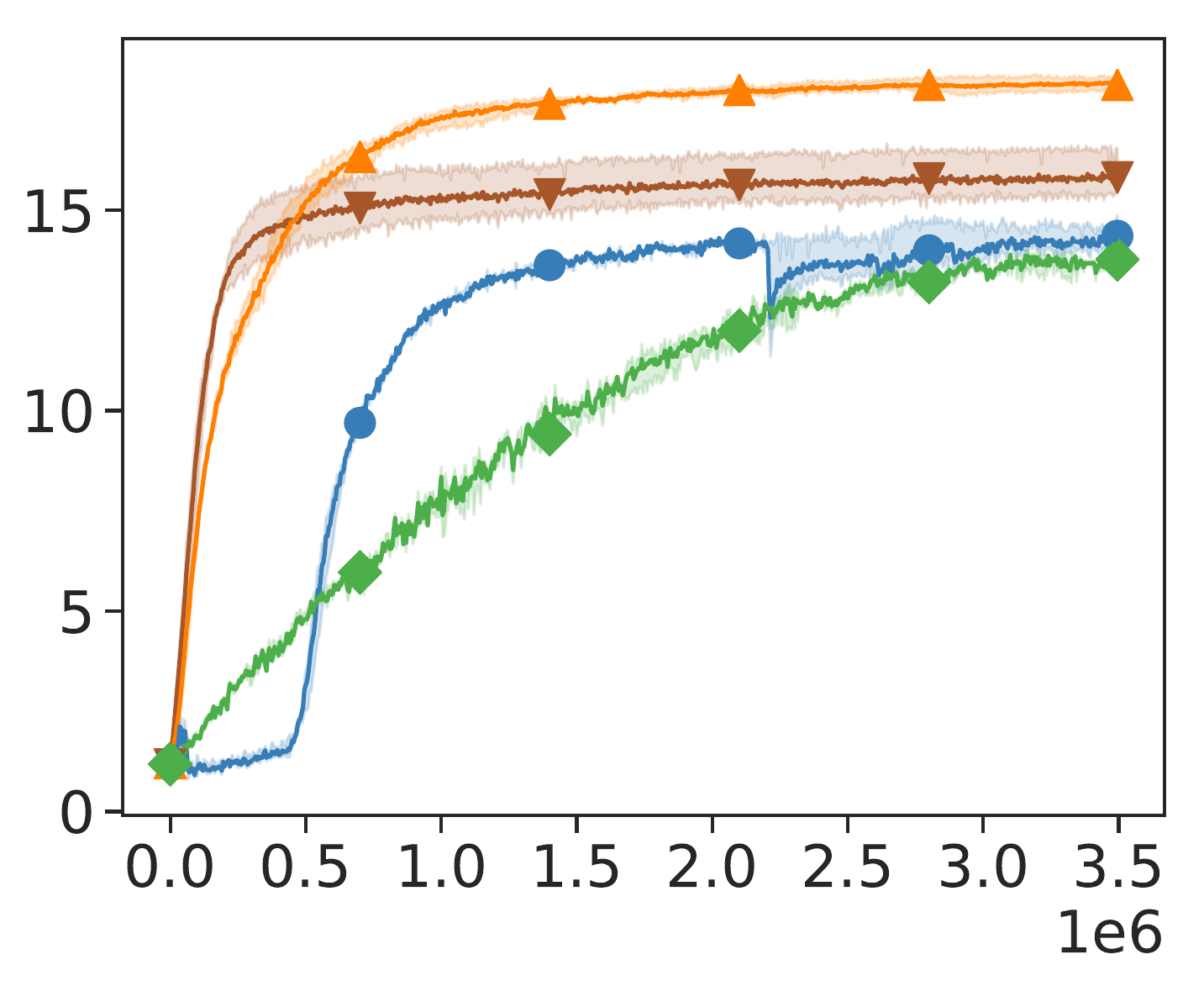}
         \caption{Reward}
         \label{fig:avegather_reward}
     \end{subfigure}
     \hfill
     \begin{subfigure}[t]{0.47\columnwidth}
         \centering
         \includegraphics[width=\textwidth]{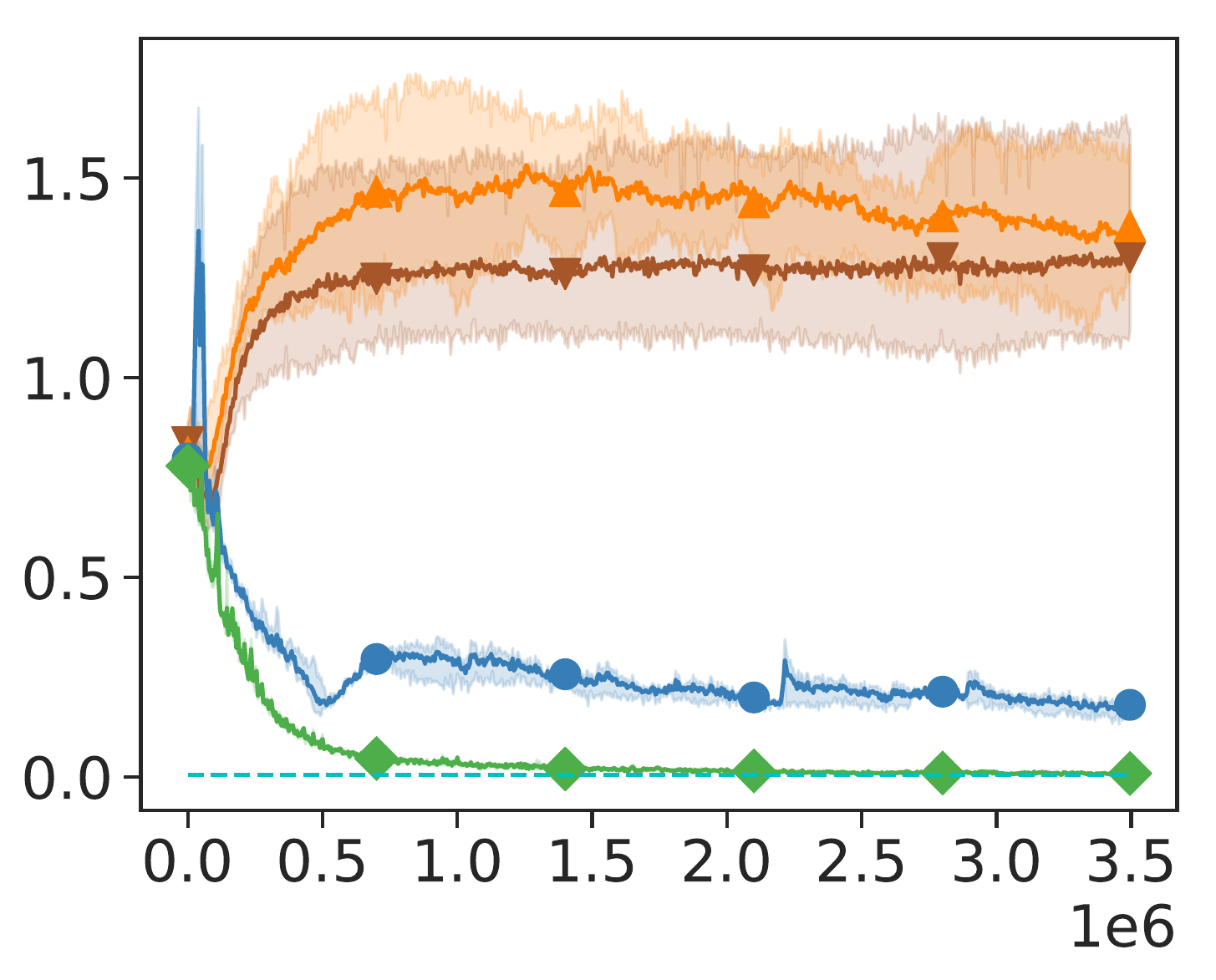}
         \caption{Constraint}
         \label{fig:avegather_cost}
     \end{subfigure}
     
     \begin{minipage}{\columnwidth}
        \centering
        Point Circle with mean valued constraint
    \end{minipage}
    \vspace{5pt}
     \begin{subfigure}[t]{0.47\columnwidth}
         \centering
         \includegraphics[width=\textwidth]{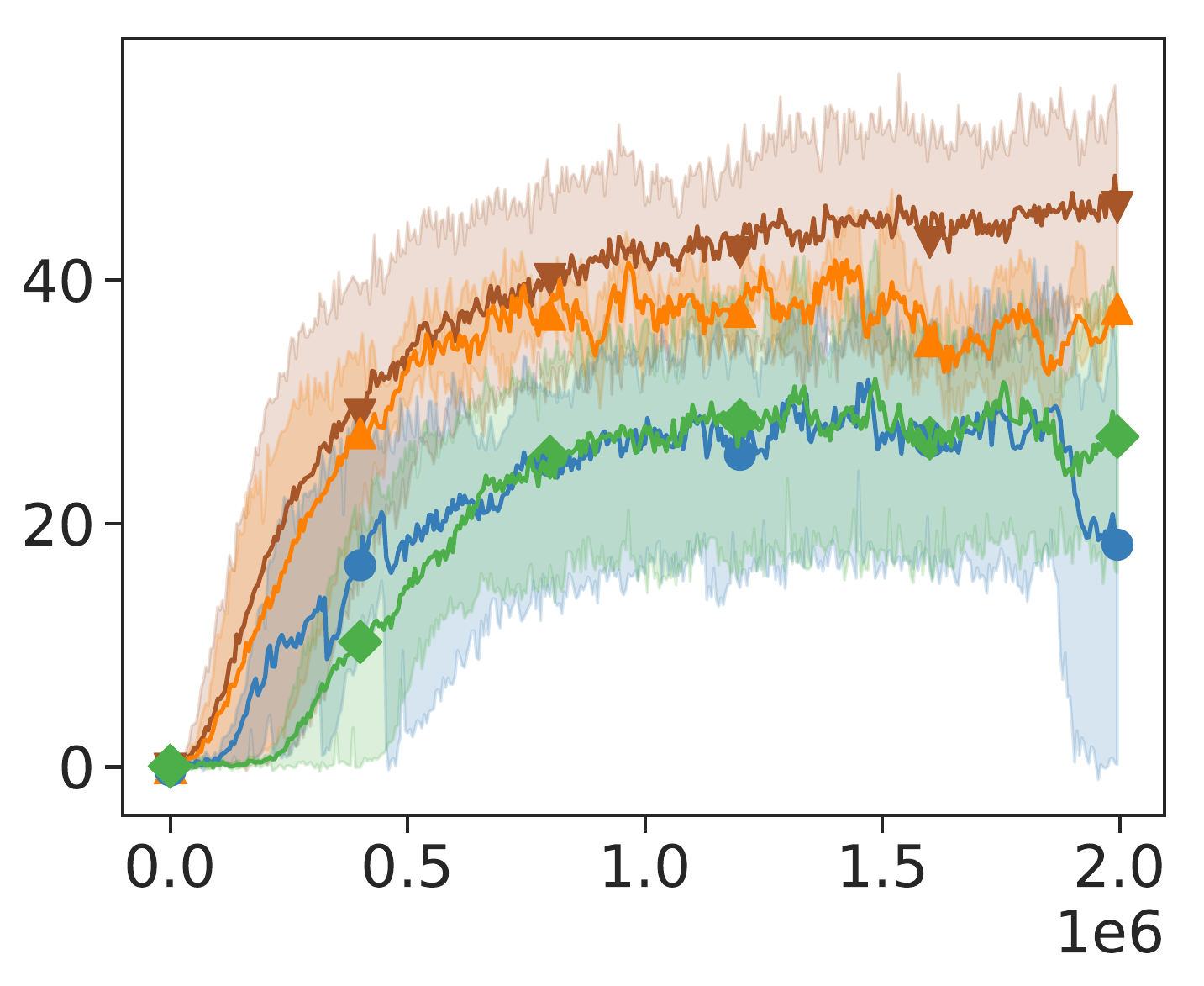}
         \caption{Reward}
         \label{fig:avecircle_reward}
     \end{subfigure}
     \hfill
     \begin{subfigure}[t]{0.47\columnwidth}
         \centering
         \includegraphics[width=\textwidth]{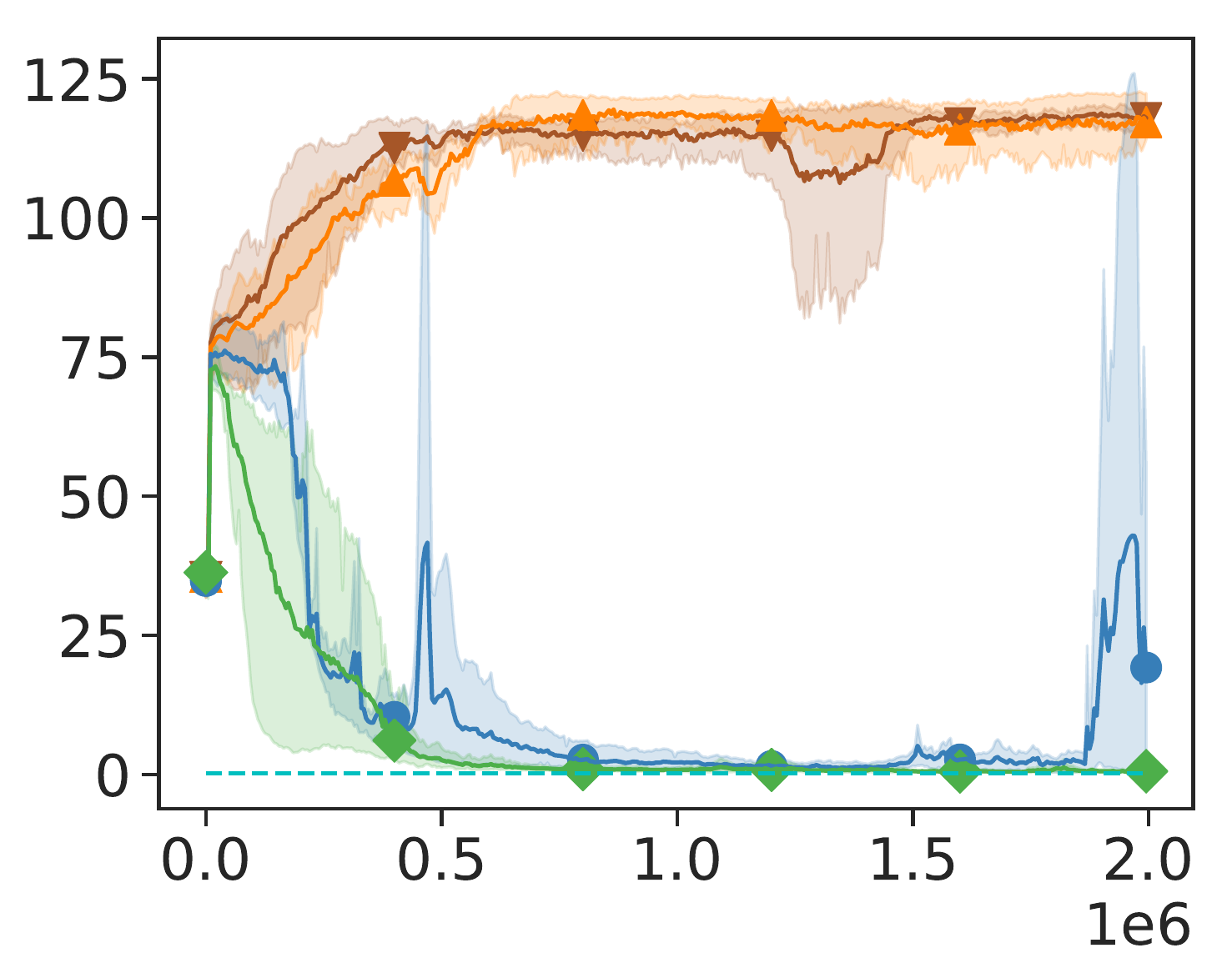}
         \caption{Constraint}
         \label{fig:avecircle_cost}
     \end{subfigure}
     
     \begin{minipage}{\columnwidth}
        \centering
        Mars Rover with mean valued constraint
    \end{minipage}
    \vspace{3pt}
     \begin{subfigure}[t]{0.47\columnwidth}
         \centering
         \includegraphics[width=\textwidth]{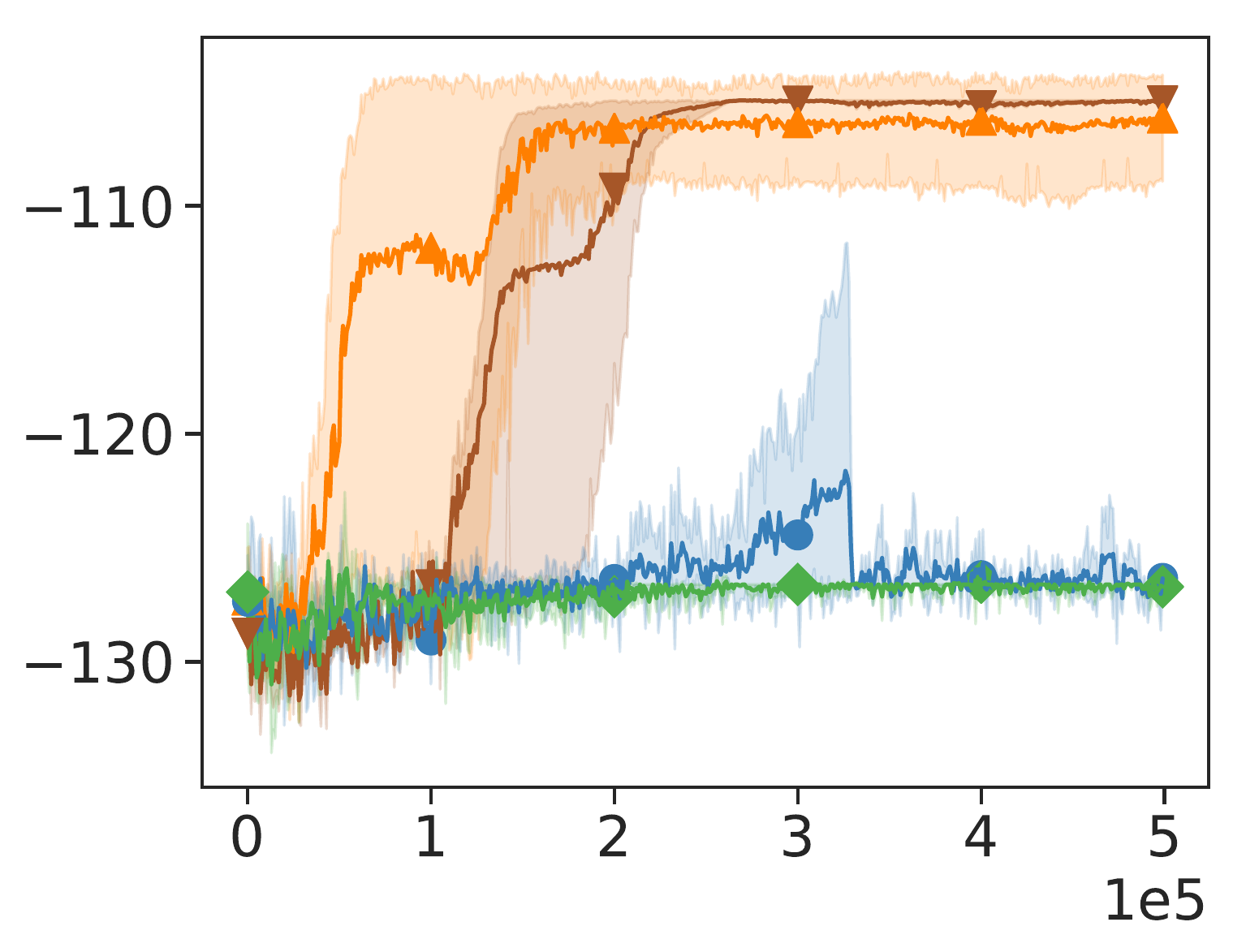}
         \caption{Reward}
         \label{fig:marsrover_reward}
     \end{subfigure}
     \hfill
     \begin{subfigure}[t]{0.47\columnwidth}
         \centering
         \includegraphics[width=\textwidth]{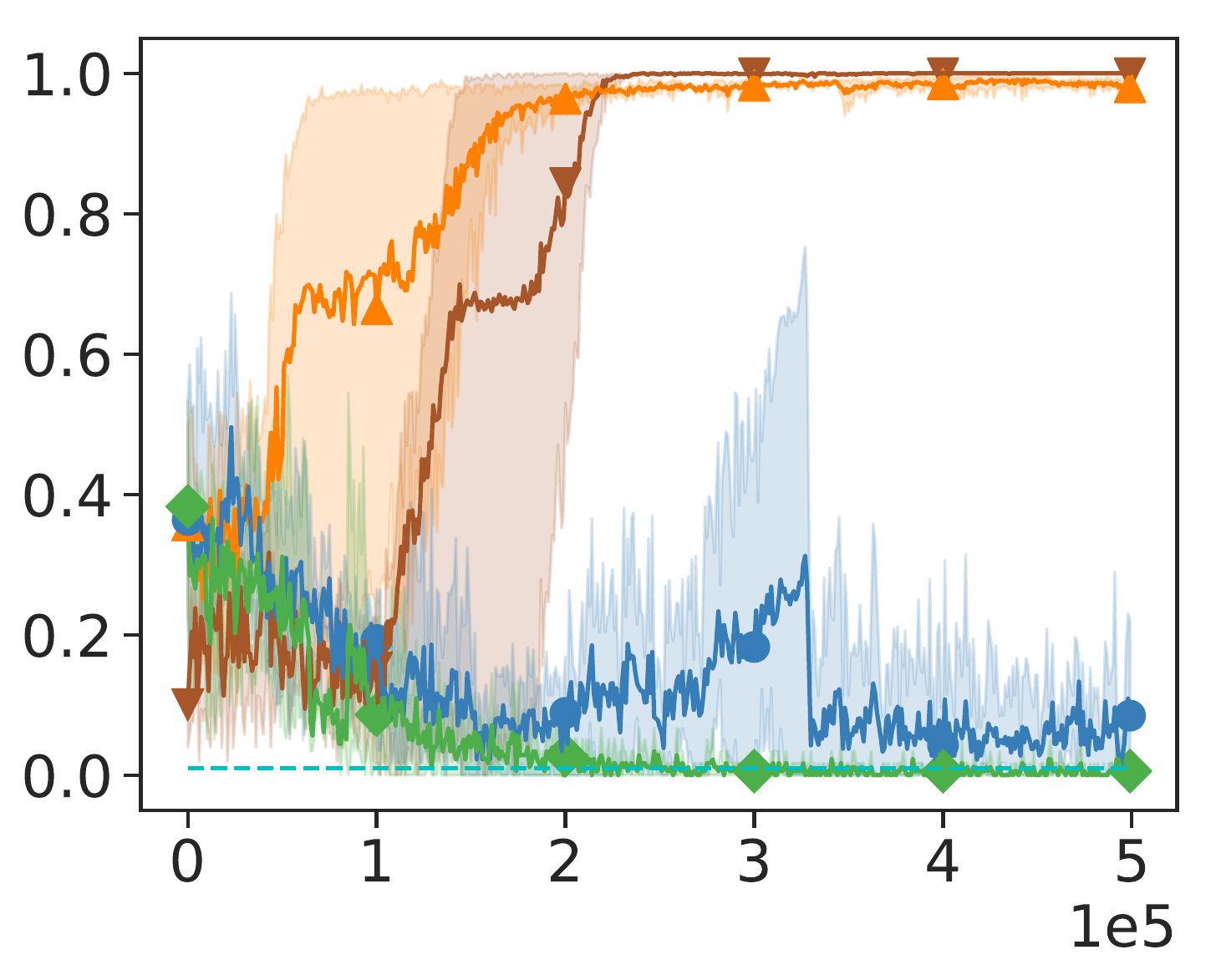}
         \caption{Constraint}
         \label{fig:marsrover_cost}
     \end{subfigure}
     
     \begin{minipage}[t]{0.9\columnwidth}
         \centering
         \includegraphics[width=\textwidth]{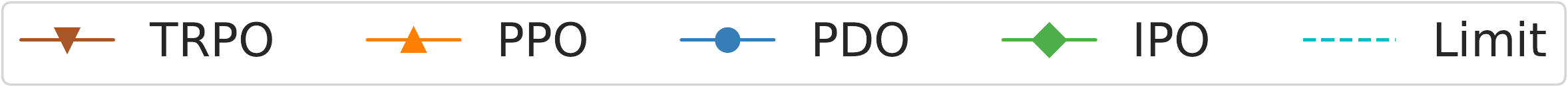}
    \end{minipage}
     
     \caption{Average performance of TRPO, PPO, PDO and IPO under Point Gather, Point Circle and Mars Rover with mean valued constraints.
     The dashed lines are constrained limits for different tasks which is $0.005$ for Point Gather, $0.2$ for Point Circle and $0.01$ for Mars Rover.}
     \label{fig:mean}
\end{figure}

\section{Experiments}
In the experiment, we demonstrate the following properties of IPO: 
\begin{itemize}
    \item IPO can handle more general types of cumulative constraints including discounted cumulative constraints and mean valued constraints. It outperforms the state-of-the-art baselines, CPO~\cite{achiam2017constrained} and PDO~\cite{chow2017risk}, on both constraints.
    \item  IPO's hyperparameter is easy to tune, compared to PDO.
    \item IPO can be easily extended to handle optimizations with multiple constraints.
    \item IPO is robust in stochastic environments.  
\end{itemize}

We conduct experiments and compare IPO with CPO and PDO in various scenarios: three tasks in the Mujoco simulator (Point-Gather, Point-Circle~\cite{achiam2017constrained}, HalfCheetah-Safe~\cite{chow2019lyapunov}) and a grid-world task (Mars-Rover) inspired by~\cite{chow2015risk}. 

To be fair, the baseline algorithms inherit all advantages in IPO. For example, PDO is implemented with the same PPO clipped surrogate objective. Additionally, we demonstrate the performance of PPO and TRPO for reference, although these algorithms do not take constraints into consideration. 
We run each experiment ten times with different random seeds and show the average performance.


\subsection{Scenario Description}
We first describe our experiment scenarios. 
The objective in all following scenarios is to maximize the reward (the higher the better) while satisfying the constraints (the lower the better).  
\begin{itemize}
\item Point-Gather: A Point agent moves in a fixed square region where there are randomly located apples ($2$ apples) and bombs ($8$ bombs). The agent is rewarded for collecting apples and there is a constraint on the number of bombs collected;
\item Point-Circle: A Point agent moves in a circular region with radius $r$. It's rewarded for running counter-clockwise along the circle, but is restricted to stay within a safe region, smaller than the circular region;
\item HalfCheetah-Safe: The HalfCheetah agent ($2$-legged simulated robot) is rewarded for running  with a speed limit at $1$, for stability and safety;
\item Grid-world: A rover travels in a fixed square region. It starts from the top left corner and its destination is the top right corner. The rover gets a negative reward for each step it moves.  There are fixed holes in the region. If the rover falls into a hole, the trip terminates.  
The constraint is on the possibility of the rover falling into a hole.  
\end{itemize}

\subsection{Discounted Cumulative Constraints}\label{exp:discounted}
First, 
we demonstrate results on optimization with discounted cumulative constraints (Eq. (\ref{discounted})) on three Mujoco tasks (Point-Gather, Point-Circle, and HalfCheetah-Safe) to compare IPO with CPO and PDO.

In Figure~\ref{fig:discounted}, we can see our overall performance is better than CPO and PDO under all the three Mujoco tasks.
IPO achieves higher discounted cumulative reward with lower discounted cumulative cost than CPO. 
CPO converges faster than IPO, but we notice that CPO always stops improving when the constraint is satisfied. 
On the contrary, IPO continues to search for a better policy even if the constraint is satisfied. 
Hence, it converges to a better reward and lower cost. 

For PDO, we try to find the satisfactory initial Lagrange multiplier and learning rate with grid search, which is time-consuming. 
From Figure ~\ref{fig:gather_reward} and \ref{fig:gather_cost}, we can see that PDO can converge to a policy as good as IPO, however, the variance of the performance during training is high.  
The performance of PDO in Figure~\ref{fig:circle_reward} and \ref{fig:circle_cost} is worst of all. 
It indicates that the Lagrange multiplier $0.01$ and learning rate $0.01$ is a bad initialization. In the HalfCheetah-Safe task, Figure~\ref{fig:halfcheetah_reward} and \ref{fig:halfcheetah_cost}, PDO achieves a policy whose constraint value is lower than the limit, but the reward is the lowest as well. PDO is sensitive to the initialization of the Lagrange multiplier and learning rate. We will also demonstrate the impact of hyperparameters in the following experiment.

The PPO and TRPO  consider the optimization without constraints. They achieve higher rewards as well as violating the constraints more,   
compared to IPO, CPO and PDO. 

\begin{figure}[t]
    \centering
    
     \begin{subfigure}[t]{0.47\columnwidth}
         \centering
         \includegraphics[width=\textwidth]{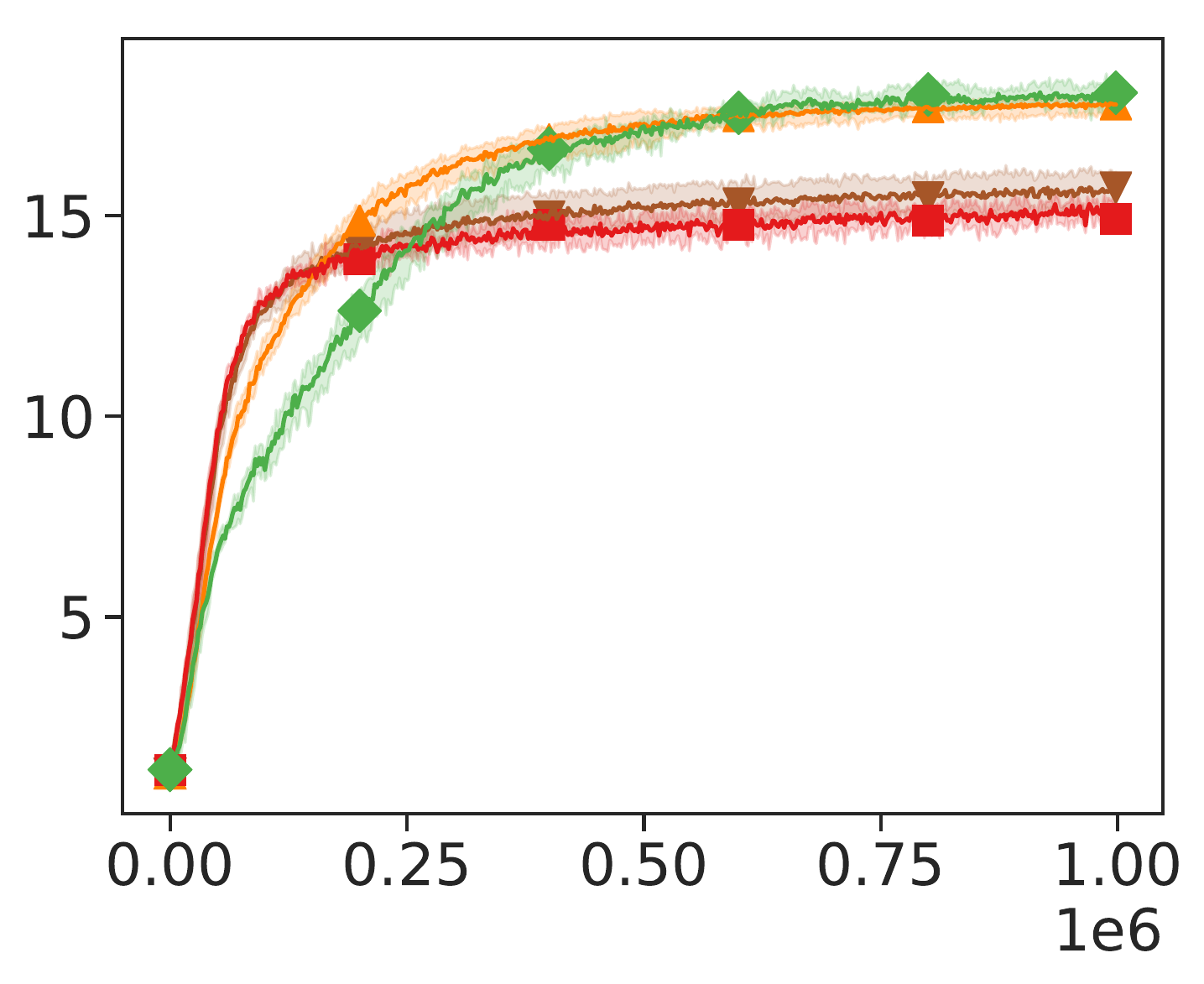}
         \caption{Reward}
         \label{fig:vscpo_reward}
     \end{subfigure}
     \hfill
     \vspace{3pt}
     \begin{subfigure}[t]{0.47\columnwidth}
         \centering
         \includegraphics[width=\textwidth]{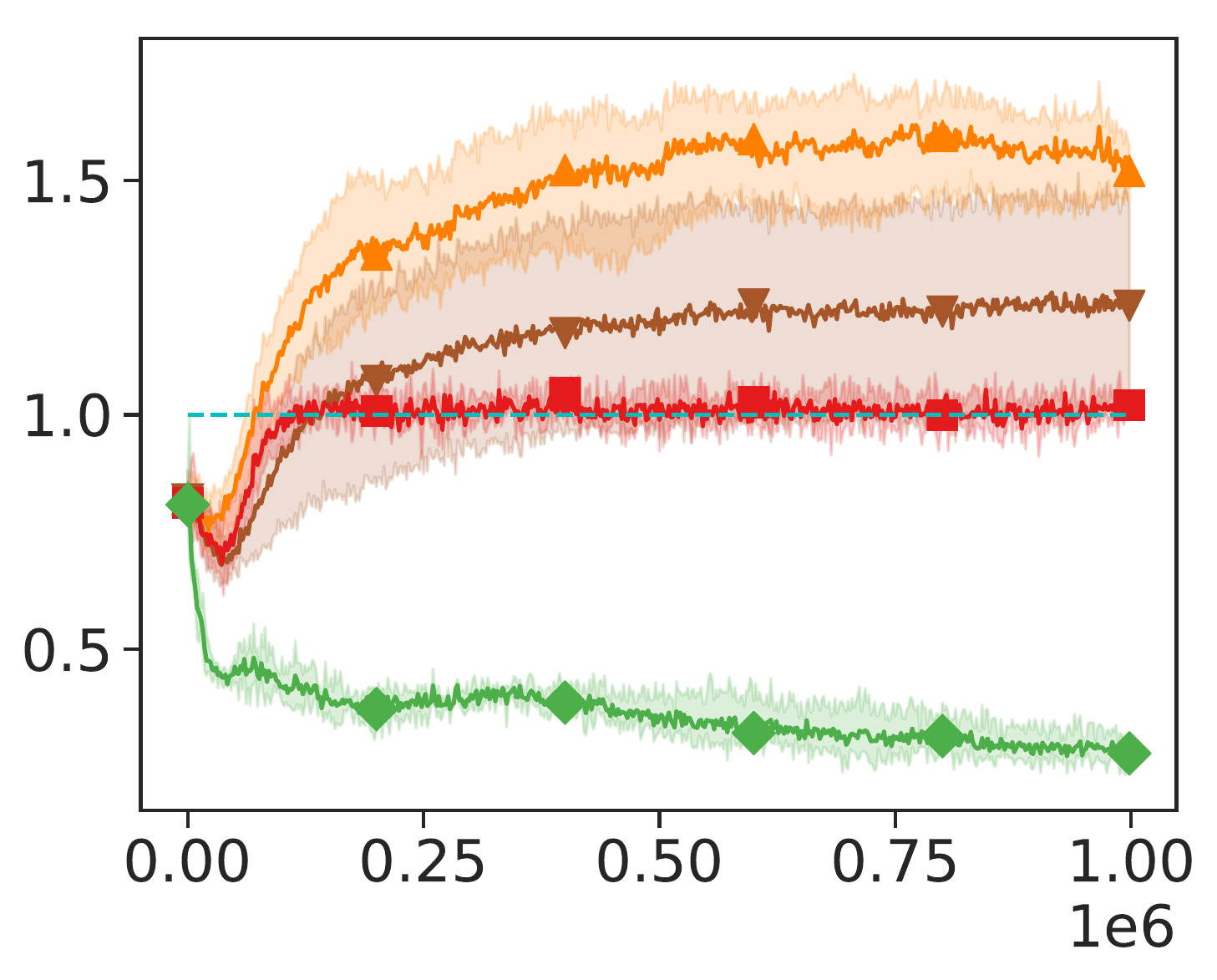}
         \caption{Constraint}
         \label{fig:vscpo_cost}
     \end{subfigure}
    
     \begin{minipage}[t]{0.9\columnwidth}
         \centering
         \includegraphics[width=\textwidth]{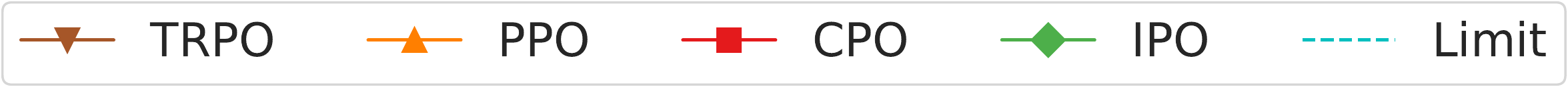}
    \end{minipage}
     
     \caption{Average performance of TRPO, PPO, CPO and IPO under constraint limit 1.}
     \label{fig:vscpo}
\end{figure}

     
     

\subsection{Mean Valued Constraints}\label{exp:mean}

Our algorithm can not only support discounted cumulative constraints but also mean valued constraints (Eq. (\ref{mean})).

We conduct experiments on optimization with the mean valued constrains on two Mujoco tasks (Point-Gather, Point-Circle) and the grid-world (Mars-Rover).  Because the CPO does not support mean valued constraints. We only compare IPO with PDO.

Figure~\ref{fig:mean} shows that IPO can consistently converge to a policy with high discounted cumulative reward and satisfy the mean valued constrains on all tasks. PDO, however, sometimes converges to a policy violating the constraints (Figure~\ref{fig:avegather_cost}) and has a higher variance during training  (Figure~\ref{fig:avecircle_cost} and Figure~\ref{fig:marsrover_cost}). 



\begin{figure}[t]
    \centering
     \begin{subfigure}[t]{0.47\columnwidth}
         \centering
         \includegraphics[width=\textwidth]{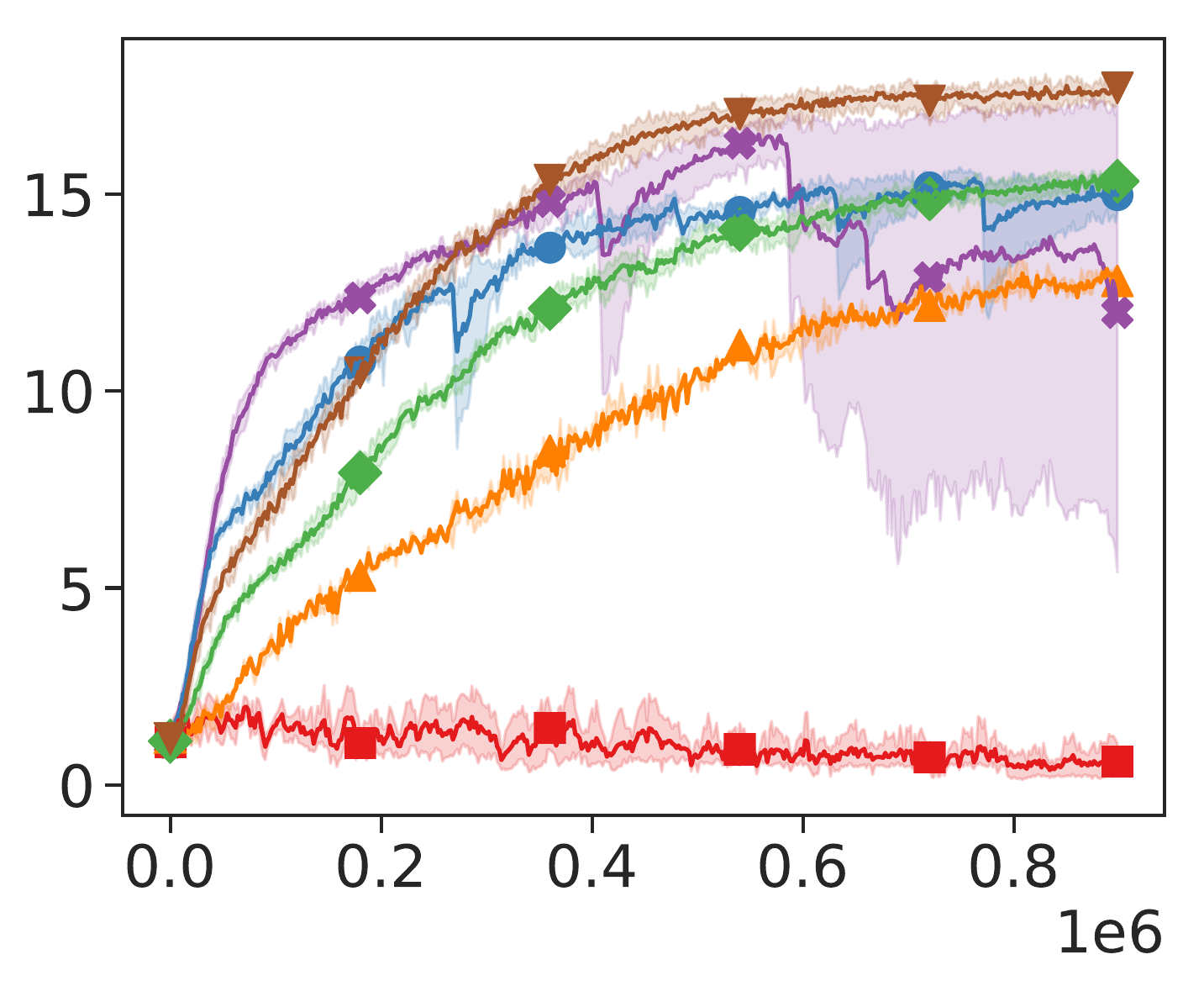}
         \caption{Reward}
         \label{fig:vspdo_reward}
     \end{subfigure}
     \hfill
     \vspace{3pt}
     \begin{subfigure}[t]{0.47\columnwidth}
         \centering
         \includegraphics[width=\textwidth]{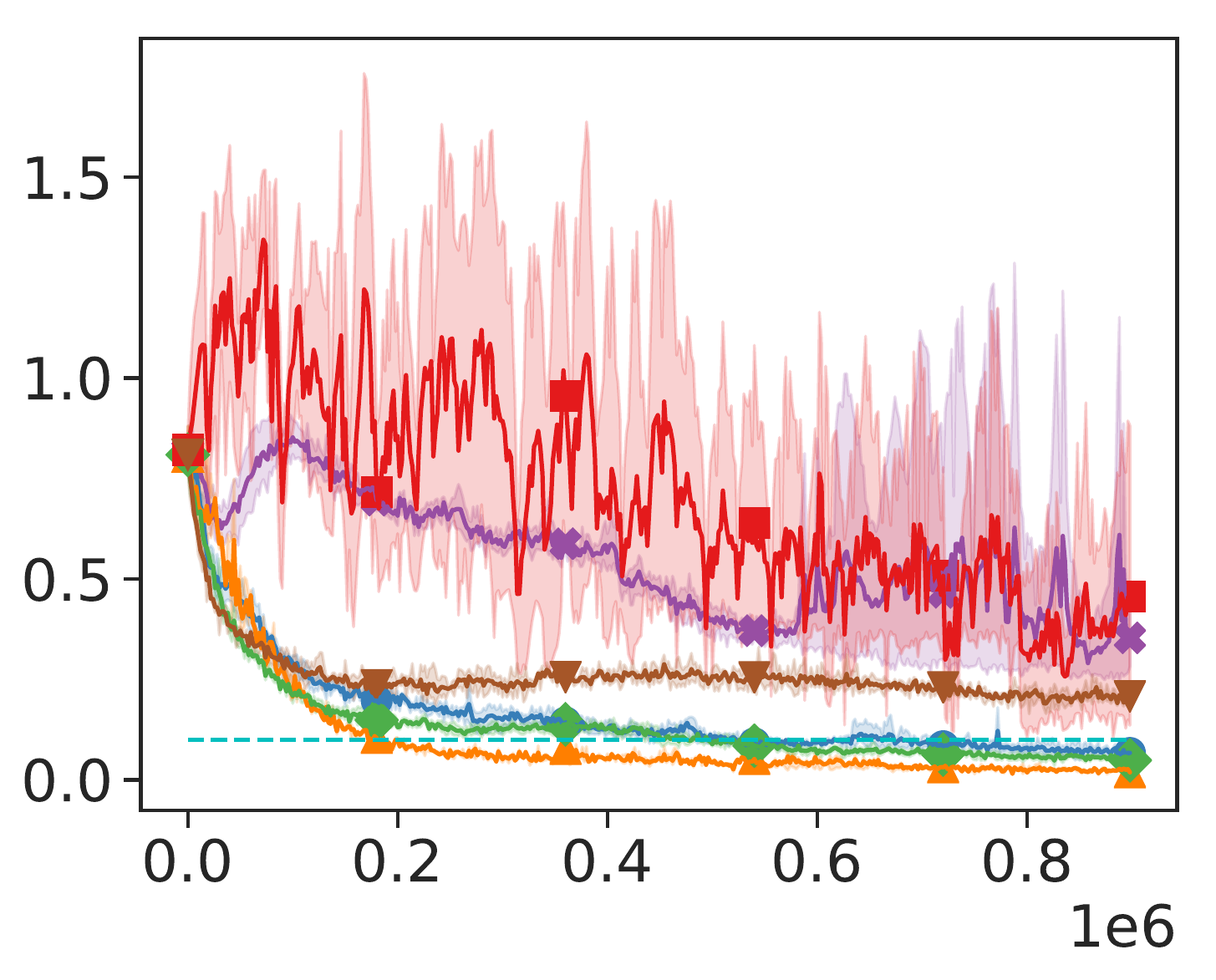}
         \caption{Constraint}
         \label{fig:vspdo_cost}
     \end{subfigure}
     
     \begin{minipage}[t]{0.9\columnwidth}
         \centering
         \includegraphics[width=\textwidth]{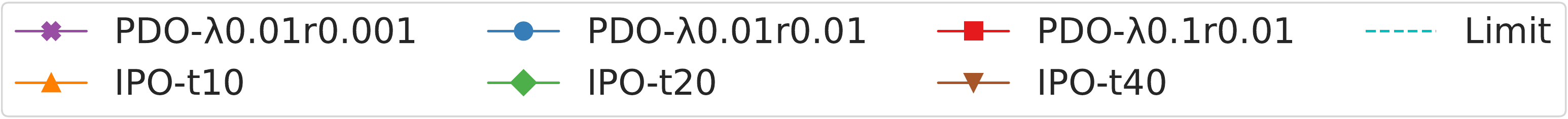}
    \end{minipage}
     
     \caption{Average performance of PDO and IPO with different hyperparameters.}
     \label{fig:vspdo}
\end{figure}

\begin{figure*}[t]
     \centering
     \begin{subfigure}[t]{0.3\textwidth}
         \centering
         \includegraphics[width=\textwidth]{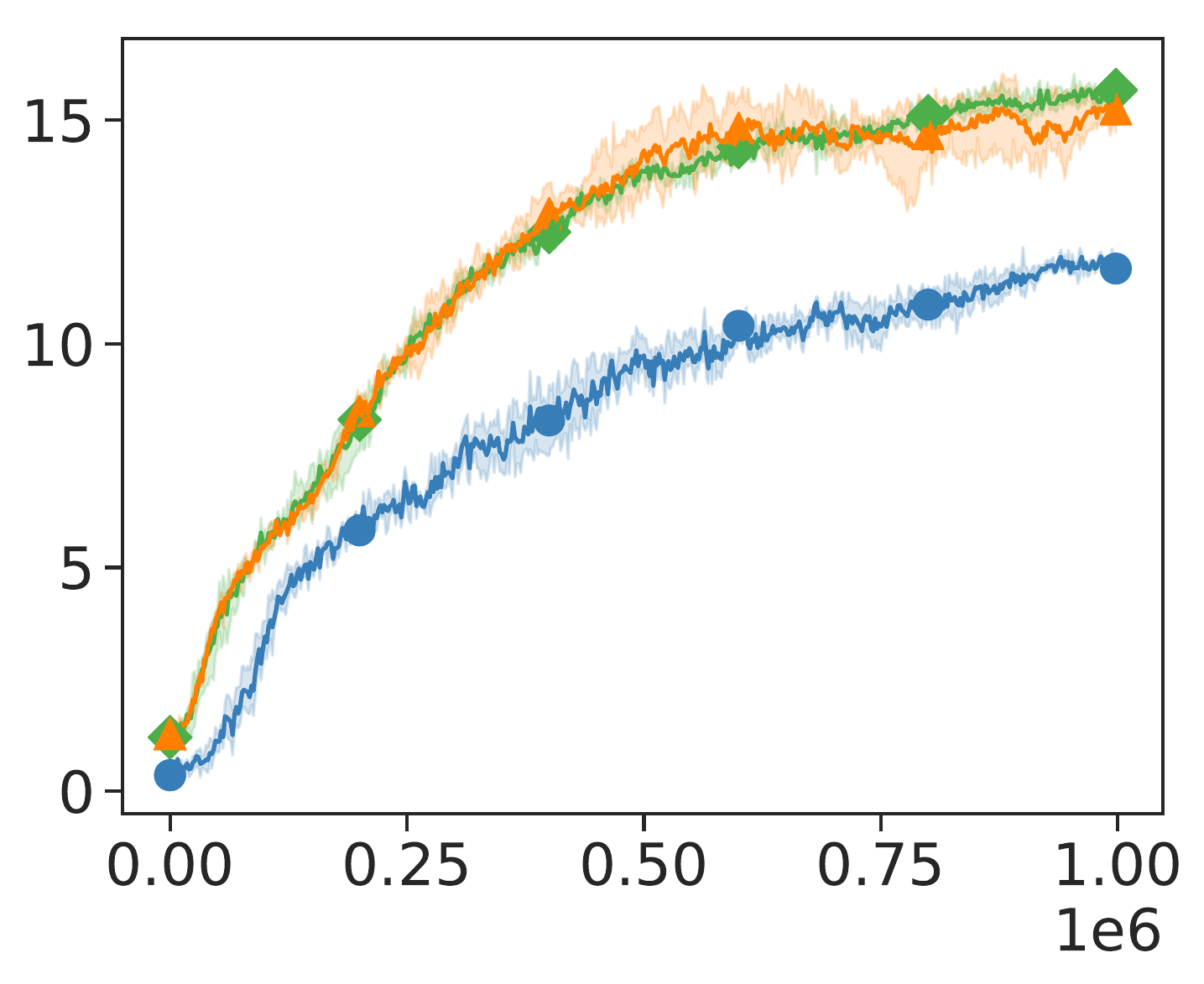}
         \caption{Reward}
         \label{fig:multi_reward}
     \end{subfigure}
     \hfill
     \begin{subfigure}[t]{0.3\textwidth}
         \centering
         \includegraphics[width=\textwidth]{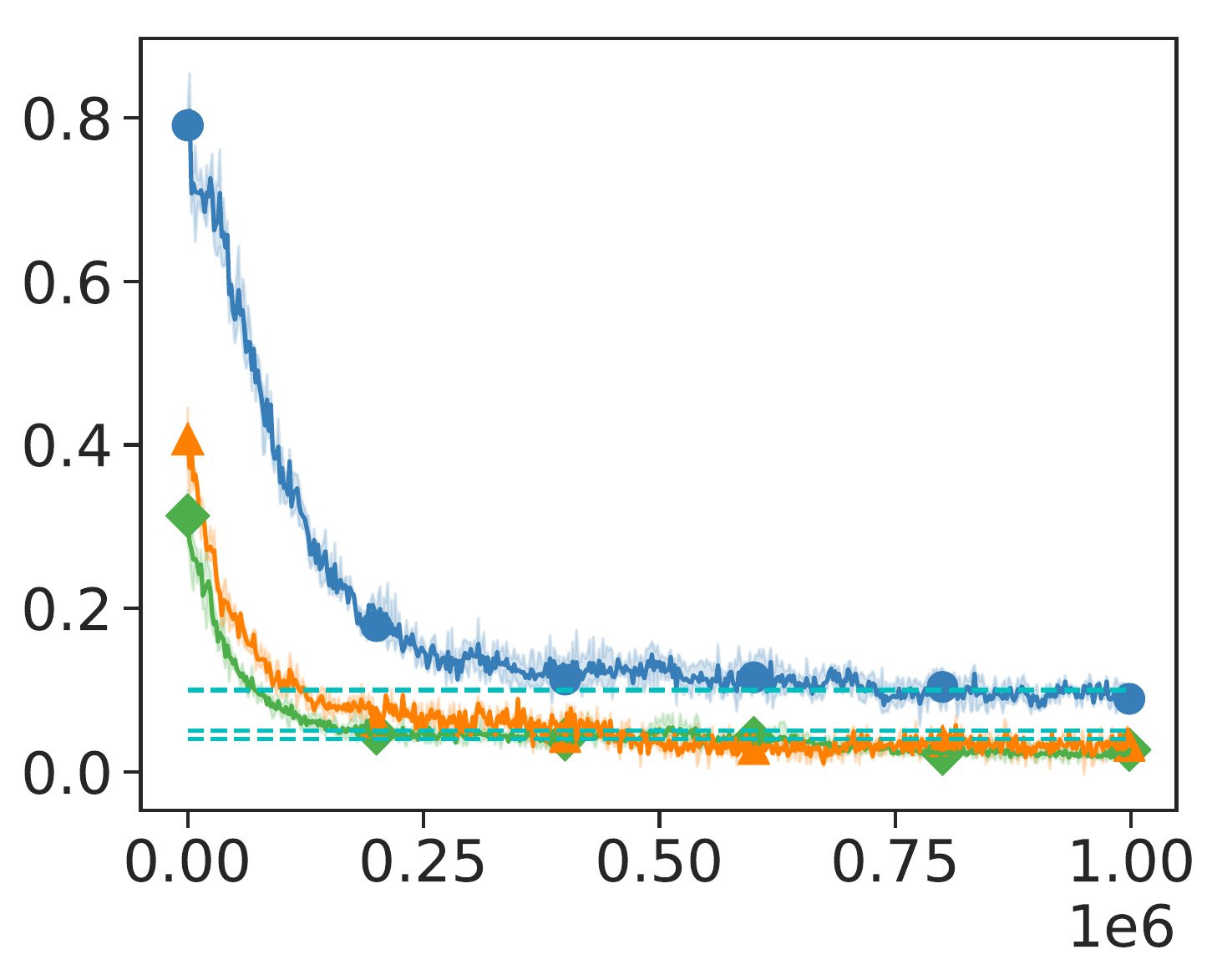}
         \caption{First constraint}
         \label{fig:multi_cost1}
     \end{subfigure}
     \hfill
     \vspace{3pt}
     \begin{subfigure}[t]{0.3\textwidth}
         \centering
         \includegraphics[width=\textwidth]{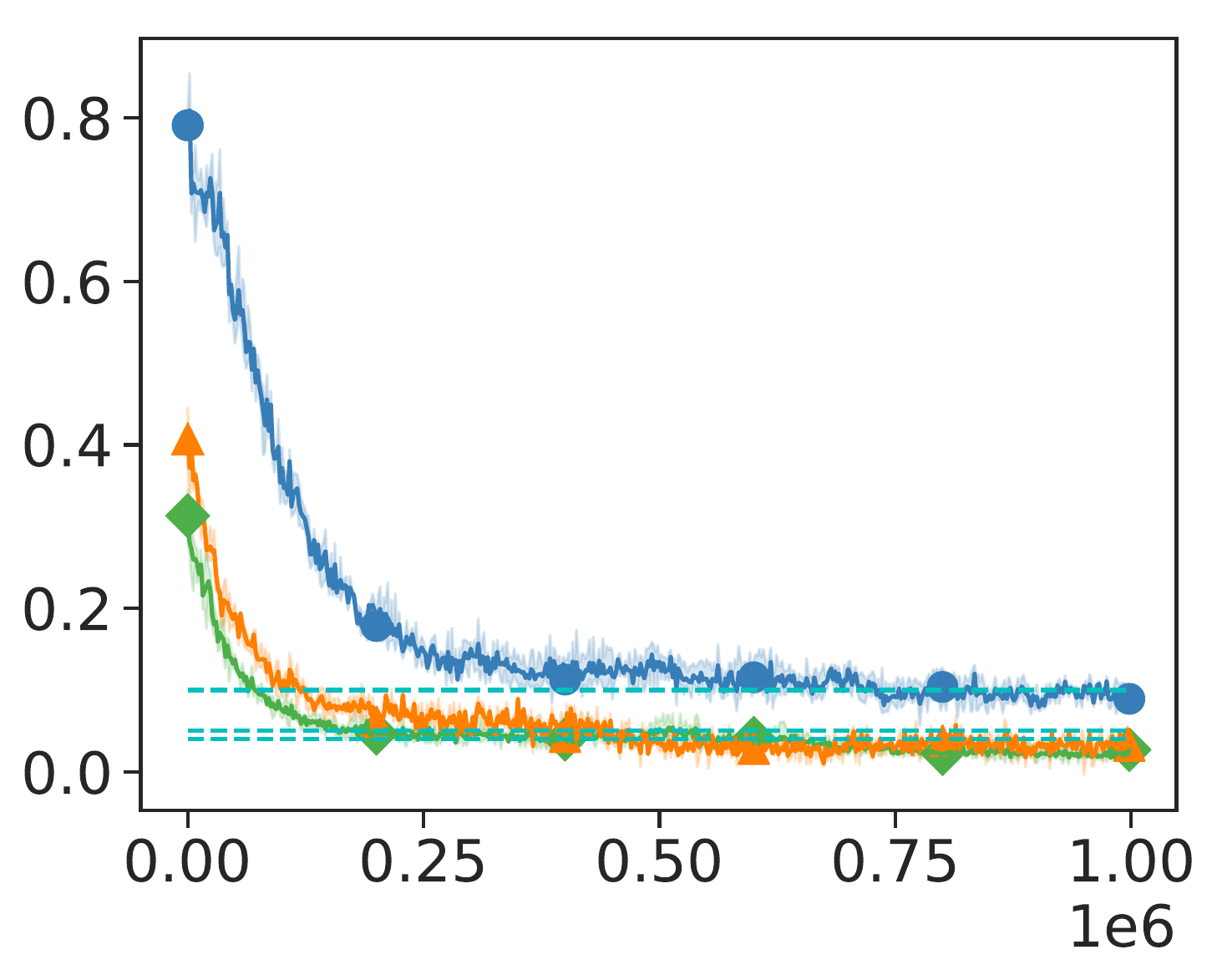}
         \caption{Second constraint}
         \label{fig:multi_cost2}
     \end{subfigure}
     \begin{minipage}[t]{0.9\textwidth}
         \centering
         \includegraphics[width=\textwidth]{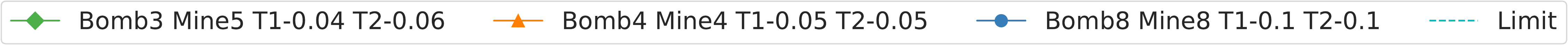}
    \end{minipage}
     \caption{Average performance of IPO under multi-constraints. T1 and T2 correspond to the the limits in (b) and (c) separately. The dash lines are limits for different task settings}
     \label{fig:multi}
\end{figure*}

\subsection{Constraint Effects}\label{exp:cpo} 

In this experiment, we would like to analyze the effects when changing the constraints. 
We loosen the constraint in Point Gather with a larger threshold, to be $1$, which means that the Point agent can collect at most one bomb on average in each play.  
Figure~\ref{fig:vscpo} demonstrates that both IPO and CPO can obtain almost the same discounted cumulative reward as their corresponding unconstrained method, PPO and TRPO. 
It indicates that such a constraint is so loose that the performance of the constrained optimization is equivalent to the unconstrained one. 
We observe that CPO still increases its cost to reach the constraint $1$, which is even worse than the randomly initialized policy (around $0.8$ at the first iteration). 
The experiment reflects that CPO always makes efforts to push its cost to the constraint threshold. On the contrary, IPO keeps decreasing its cost after the constraint is satisfied. 
Statistically, the average number of bombs collected for CPO is around $1$ and the number for IPO is around $0.25$. 
We also attach the video visualizing the actions of policies learned with IPO and CPO playing a Point Gather task with a fixed configuration of $2$ apples and $8$ bombs, 
in the supplemental materials. 

\subsection{Hyperparameter Tuning}\label{exp:pdo}
Compared to PDO, our hyperparameter $t$ is easier to tune. We conduct experiments in Point Gather. As shown in Figure~\ref{fig:vspdo}, the performance of PDO is sensitive to the initialization of the Lagrange multiplier $\lambda$ from $0.01$ to $0.1$. 
Figure \ref{fig:vspdo} also demonstrates that PDO is affected by the learning rate which changes from $0.01$ to $0.001$. The smaller learning rate slows down the policy convergence pace.

Tuning the initial Lagrange multiplier and learning rate takes a lot of efforts in PDO. 
On the contrary, the reward and cost of IPO are positively correlated with the hyperparameter $t$, which enables us to employ a binary search for a feasible hyperparameter conveniently. As shown in Figure~\ref{fig:vspdo}, we achive higher reward and cost with larger $t$. Theoretically, we start from a value $N$ big enough, and it takes us at most $O(\log(N))$ to find a feasible $t$ maximize the reward and reduce the cost to satisfy the constraints.  In practice, we usually initialize $t$ to be the maximal value of discounted cumulative reward and it can be found in a few iterations. Besides, it's fixed for each task even in different settings (e.g. different constraint limits).


\subsection{Multiple Constraints} 
IPO can be conveniently extended to the optimization with multiple constraints by adding a logarithm barrier function for each constraint, which is much easier to implement than CPO.   
In this section, we conduct three experiments in Point Gather, each with two constraints. 
To extend the task with multiple constraints, we add another type of balls,  mine balls, in the task. The cost of mine balls is the same as bomb balls, which is $1$. Now our goal is to maximize the number of apple balls collected, with constraints on the number of bomb balls and mine balls collected. Below we describe our settings, where the values in the parentheses are the constraints of the maximum expected numbers of corresponding balls collected in one play. Our settings are  
\begin{enumerate}
    \item two apples, three bomb balls ($0.04$), five mine balls ($0.06$);
    \item two apples, four bomb balls ($0.05$), four mine balls ($0.05$);
    \item two apples, eight bomb balls ($0.1$), eight mine balls ($0.1$); 
\end{enumerate}

From Figure~\ref{fig:multi}, we can see IPO can converge to a policy satisfying both constraints in all three settings while achieving a high reward. 

\subsection{Stochastic Environment Effects}

\begin{figure}[t]
     \centering
     \begin{subfigure}[t]{0.45\columnwidth}
         \centering
         \includegraphics[width=\textwidth]{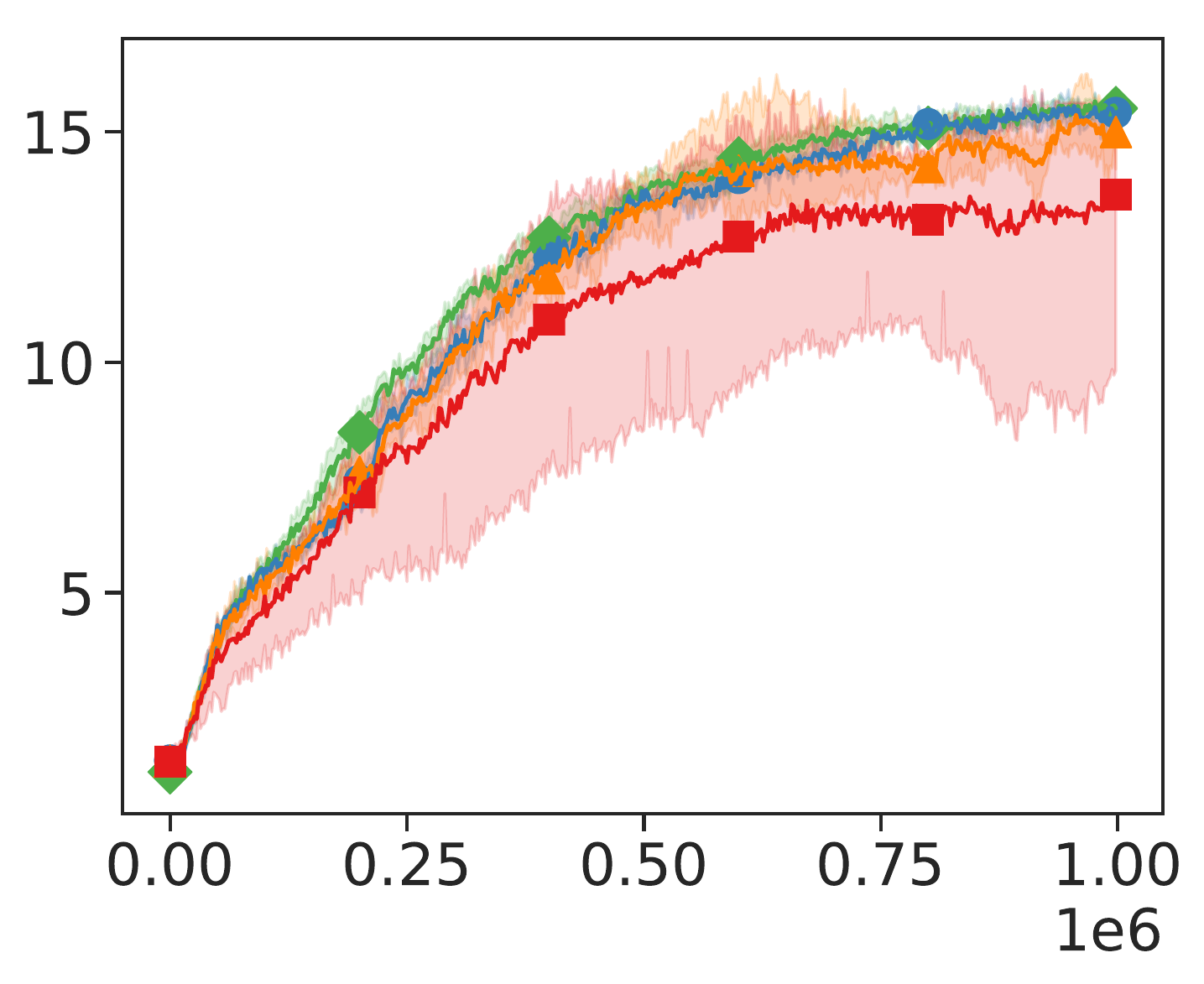}
         \caption{Reward}
         \label{fig:noise_return}
     \end{subfigure}
     \hfill
     \vspace{3pt}
     \begin{subfigure}[t]{0.45\columnwidth}
         \centering
         \includegraphics[width=\textwidth]{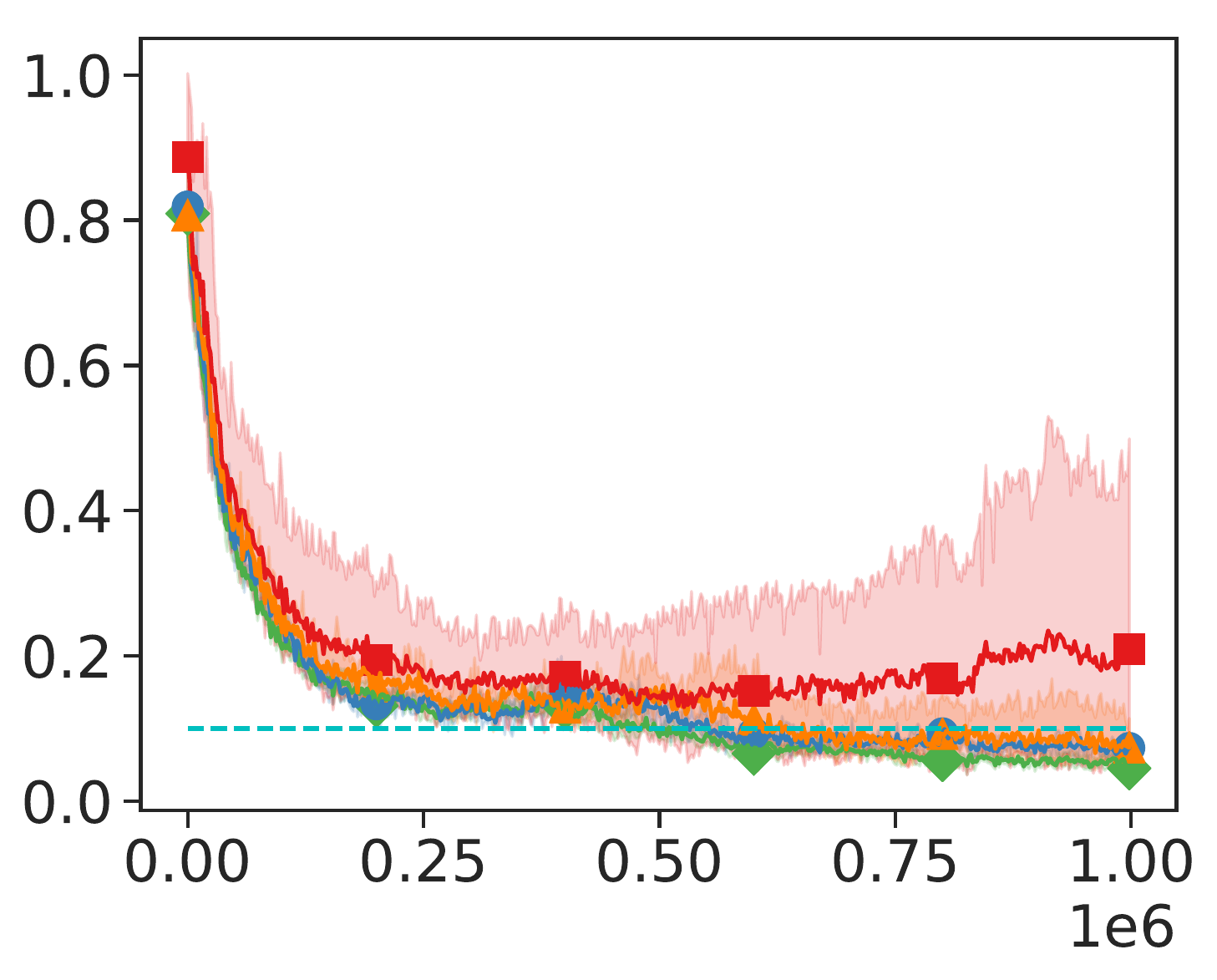}
         \caption{Constraint}
         \label{fig:noise_cost}
     \end{subfigure}
     \begin{minipage}[t]{0.9\columnwidth}
         \centering
         \includegraphics[width=\textwidth]{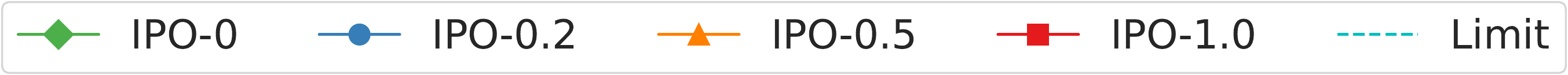}
    \end{minipage}
    \caption{Average performance of IPO under different noise scale. IPO-0 means no noise is added.}
    \label{fig:noise}
\end{figure}

All the tasks mentioned above have deterministic feedback and control. However, in real-world scenarios, there is always uncertainty from the environment. 
In this section, we do an extra experiment showing that our method is robust to a stochastic environment, where the outcome of an action is affected by random noise. 
This setting is inspired by robotics manipulation with uncertainty.

The agent's actions is represented as a vector of velocities and heading directions ranging from $-1$ to $1$. 
In the experiment, we add random noises following normal distributions with mean value $0$ and variances $0.2, 0.5, 1.0$ to the outcome of agent actions. 
By comparing the performance in Figure ~\ref{fig:noise}, we can see that 
IPO can still converge to a satisfied policy even when the scale factor is $0.5$. 


\section{Conclusion}\label{discussion}
In this paper, we propose a first-order policy optimization method, Interior-point Point Optimization (IPO),  to solve constrained reinforcement learning problems. 
Compared to the state-of-the-art methods, IPO achieves better performance and handles more general types of multiple cumulative constraints. 
In practice, IPO is easy to implement and the hyperparameters are easy to tune. In the future,
we will further provide thorough theoretical guarantees for our algorithm, and apply more optimization techniques to handle RL with constraints. 
\newpage
\bibliographystyle{aaai}
\bibliography{bibfile}

\begin{thebibliography}{}

\bibitem[\protect\citeauthoryear{Achiam \bgroup et al\mbox.\egroup
  }{2017}]{achiam2017constrained}
Achiam, J.; Held, D.; Tamar, A.; and Abbeel, P.
\newblock 2017.
\newblock Constrained policy optimization.
\newblock In {\em Proceedings of the 34th International Conference on Machine
  Learning-Volume 70},  22--31.
\newblock JMLR. org.

\bibitem[\protect\citeauthoryear{Altman}{1999}]{altman1999constrained}
Altman, E.
\newblock 1999.
\newblock {\em Constrained Markov decision processes}, volume~7.
\newblock CRC Press.

\bibitem[\protect\citeauthoryear{Andrychowicz \bgroup et al\mbox.\egroup
  }{2018}]{andrychowicz2018learning}
Andrychowicz, M.; Baker, B.; Chociej, M.; Jozefowicz, R.; McGrew, B.; Pachocki,
  J.; Petron, A.; Plappert, M.; Powell, G.; Ray, A.; et~al.
\newblock 2018.
\newblock Learning dexterous in-hand manipulation.
\newblock {\em arXiv preprint arXiv:1808.00177}.

\bibitem[\protect\citeauthoryear{Boyd and Vandenberghe}{2004}]{boyd2004convex}
Boyd, S., and Vandenberghe, L.
\newblock 2004.
\newblock {\em Convex optimization}.
\newblock Cambridge university press.

\bibitem[\protect\citeauthoryear{Chow \bgroup et al\mbox.\egroup
  }{2015}]{chow2015risk}
Chow, Y.; Tamar, A.; Mannor, S.; and Pavone, M.
\newblock 2015.
\newblock Risk-sensitive and robust decision-making: a cvar optimization
  approach.
\newblock In {\em Advances in Neural Information Processing Systems},
  1522--1530.

\bibitem[\protect\citeauthoryear{Chow \bgroup et al\mbox.\egroup
  }{2017}]{chow2017risk}
Chow, Y.; Ghavamzadeh, M.; Janson, L.; and Pavone, M.
\newblock 2017.
\newblock Risk-constrained reinforcement learning with percentile risk
  criteria.
\newblock {\em The Journal of Machine Learning Research} 18(1):6070--6120.

\bibitem[\protect\citeauthoryear{Chow \bgroup et al\mbox.\egroup
  }{2018}]{chow2018lyapunov}
Chow, Y.; Nachum, O.; Duenez-Guzman, E.; and Ghavamzadeh, M.
\newblock 2018.
\newblock A lyapunov-based approach to safe reinforcement learning.
\newblock In {\em Advances in Neural Information Processing Systems},
  8092--8101.

\bibitem[\protect\citeauthoryear{Chow \bgroup et al\mbox.\egroup
  }{2019}]{chow2019lyapunov}
Chow, Y.; Nachum, O.; Faust, A.; Ghavamzadeh, M.; and Duenez-Guzman, E.
\newblock 2019.
\newblock Lyapunov-based safe policy optimization for continuous control.
\newblock {\em arXiv preprint arXiv:1901.10031}.

\bibitem[\protect\citeauthoryear{Dalal \bgroup et al\mbox.\egroup
  }{2018}]{dalal2018safe}
Dalal, G.; Dvijotham, K.; Vecerik, M.; Hester, T.; Paduraru, C.; and Tassa, Y.
\newblock 2018.
\newblock Safe exploration in continuous action spaces.
\newblock {\em arXiv preprint arXiv:1801.08757}.

\bibitem[\protect\citeauthoryear{Dulac-Arnold, Mankowitz, and
  Hester}{2019}]{dulac2019challenges}
Dulac-Arnold, G.; Mankowitz, D.; and Hester, T.
\newblock 2019.
\newblock Challenges of real-world reinforcement learning.
\newblock {\em arXiv preprint arXiv:1904.12901}.

\bibitem[\protect\citeauthoryear{Garc{\i}a and
  Fern{\'a}ndez}{2015}]{garcia2015comprehensive}
Garc{\i}a, J., and Fern{\'a}ndez, F.
\newblock 2015.
\newblock A comprehensive survey on safe reinforcement learning.
\newblock {\em Journal of Machine Learning Research} 16(1):1437--1480.

\bibitem[\protect\citeauthoryear{Julian \bgroup et al\mbox.\egroup
  }{2002}]{julian2002qos}
Julian, D.; Chiang, M.; O'Neill, D.; and Boyd, S.
\newblock 2002.
\newblock Qos and fairness constrained convex optimization of resource
  allocation for wireless cellular and ad hoc networks.
\newblock In {\em Proceedings. Twenty-First Annual Joint Conference of the IEEE
  Computer and Communications Societies}, volume~2,  477--486.
\newblock IEEE.

\bibitem[\protect\citeauthoryear{Khalil}{2002}]{khalil2002nonlinear}
Khalil, H.~K.
\newblock 2002.
\newblock Nonlinear systems.
\newblock {\em Upper Saddle River}.

\bibitem[\protect\citeauthoryear{Kullback and
  Leibler}{1951}]{kullback1951information}
Kullback, S., and Leibler, R.~A.
\newblock 1951.
\newblock On information and sufficiency.
\newblock {\em The annals of mathematical statistics} 22(1):79--86.

\bibitem[\protect\citeauthoryear{Le, Voloshin, and Yue}{2019}]{le2019batch}
Le, H.~M.; Voloshin, C.; and Yue, Y.
\newblock 2019.
\newblock Batch policy learning under constraints.
\newblock {\em arXiv preprint arXiv:1903.08738}.

\bibitem[\protect\citeauthoryear{Liang, Que, and
  Modiano}{2018}]{liang2018accelerated}
Liang, Q.; Que, F.; and Modiano, E.
\newblock 2018.
\newblock Accelerated primal-dual policy optimization for safe reinforcement
  learning.
\newblock {\em arXiv preprint arXiv:1802.06480}.

\bibitem[\protect\citeauthoryear{Mnih \bgroup et al\mbox.\egroup
  }{2015}]{mnih2015human}
Mnih, V.; Kavukcuoglu, K.; Silver, D.; Rusu, A.~A.; Veness, J.; Bellemare,
  M.~G.; Graves, A.; Riedmiller, M.; Fidjeland, A.~K.; Ostrovski, G.; et~al.
\newblock 2015.
\newblock Human-level control through deep reinforcement learning.
\newblock {\em Nature} 518(7540):529.

\bibitem[\protect\citeauthoryear{Neely}{2010}]{neely2010stochastic}
Neely, M.~J.
\newblock 2010.
\newblock Stochastic network optimization with application to communication and
  queueing systems.
\newblock {\em Synthesis Lectures on Communication Networks} 3(1):1--211.

\bibitem[\protect\citeauthoryear{Pham, De~Magistris, and
  Tachibana}{2018}]{pham2018optlayer}
Pham, T.-H.; De~Magistris, G.; and Tachibana, R.
\newblock 2018.
\newblock Optlayer-practical constrained optimization for deep reinforcement
  learning in the real world.
\newblock In {\em 2018 IEEE International Conference on Robotics and Automation
  (ICRA)},  6236--6243.
\newblock IEEE.

\bibitem[\protect\citeauthoryear{Schulman \bgroup et al\mbox.\egroup
  }{2015}]{schulman2015trust}
Schulman, J.; Levine, S.; Abbeel, P.; Jordan, M.; and Moritz, P.
\newblock 2015.
\newblock Trust region policy optimization.
\newblock In {\em International conference on machine learning},  1889--1897.

\bibitem[\protect\citeauthoryear{Schulman \bgroup et al\mbox.\egroup
  }{2017}]{schulman2017proximal}
Schulman, J.; Wolski, F.; Dhariwal, P.; Radford, A.; and Klimov, O.
\newblock 2017.
\newblock Proximal policy optimization algorithms.
\newblock {\em arXiv preprint arXiv:1707.06347}.

\bibitem[\protect\citeauthoryear{Silver \bgroup et al\mbox.\egroup
  }{2016}]{silver2016mastering}
Silver, D.; Huang, A.; Maddison, C.~J.; Guez, A.; Sifre, L.; Van Den~Driessche,
  G.; Schrittwieser, J.; Antonoglou, I.; Panneershelvam, V.; Lanctot, M.;
  et~al.
\newblock 2016.
\newblock Mastering the game of go with deep neural networks and tree search.
\newblock {\em nature} 529(7587):484.

\bibitem[\protect\citeauthoryear{Sutton and
  Barto}{2018}]{sutton2018reinforcement}
Sutton, R.~S., and Barto, A.~G.
\newblock 2018.
\newblock {\em Reinforcement learning: An introduction}.
\newblock MIT press.

\bibitem[\protect\citeauthoryear{Sutton \bgroup et al\mbox.\egroup
  }{2000}]{sutton2000policy}
Sutton, R.~S.; McAllester, D.~A.; Singh, S.~P.; and Mansour, Y.
\newblock 2000.
\newblock Policy gradient methods for reinforcement learning with function
  approximation.
\newblock In {\em Advances in neural information processing systems},
  1057--1063.

\bibitem[\protect\citeauthoryear{Tessler, Mankowitz, and
  Mannor}{2018}]{tessler2018reward}
Tessler, C.; Mankowitz, D.~J.; and Mannor, S.
\newblock 2018.
\newblock Reward constrained policy optimization.
\newblock {\em arXiv preprint arXiv:1805.11074}.

\end{thebibliography}
\end{document}